\documentclass{article}

\usepackage[preprint]{neurips_2026}

\usepackage[utf8]{inputenc} 
\usepackage[T1]{fontenc}    
\usepackage{hyperref}       
\usepackage{url}            
\usepackage{booktabs}       
\usepackage{amsfonts}       
\usepackage{nicefrac}       
\usepackage{microtype}      
\usepackage{xcolor}         
\usepackage{graphicx} 
\usepackage{amsmath} 

\usepackage{microtype}
\usepackage{graphicx}
\usepackage{subcaption}
\usepackage{booktabs} 

\usepackage{hyperref}
\usepackage{wrapfig}

\usepackage{amsmath}
\usepackage{amssymb}
\usepackage{mathtools}
\usepackage{amsthm}

\usepackage{algorithm}
\usepackage{algorithmic}

\usepackage{multirow}

\usepackage{makecell}
\usepackage{tabularx}

\usepackage[capitalize,noabbrev]{cleveref}

\theoremstyle{plain}
\newtheorem{theorem}{Theorem}[section]

\newtheorem{lemma}[theorem]{Lemma}

\theoremstyle{definition}
\newtheorem{definition}[theorem]{Definition}

\theoremstyle{remark}

\usepackage{enumitem}
\newlist{myitems}{enumerate}{1}
\setlist[myitems, 1]
{label=\arabic{myitemsi}., 
leftmargin=15pt,
rightmargin=10pt
}

\title{Voronoi Histograms for Adaptive Vectorization of Expected Persistence Diagrams}

\author{%
  Kaifeng Zhang \\
  School of Artificial Intelligence\\
  Nanjing University, China\\
\texttt{zhangkf2022@lamda.nju.edu.cn} 
  \And
  Kai Ming Ting \\
  School of Artificial Intelligence\\
  Nanjing University, China\\
\texttt{tingkm@nju.edu.cn} 
}

\begin{document}

\maketitle

\begin{abstract}
Persistence Diagram (PD) is known to capture point cloud topology effectively, but its computation has high time complexity. Expected Persistence Diagram (EPD) has been developed to reduce the time cost by studying the topology of multiple subsets of a point cloud and it serves as a distribution of topological features. Existing EPD vectorizations often rely on predefined point transformations, such as Gaussian or landscape functions. We study an alternative discretization based on Voronoi histograms, which trades smooth functional approximation for adaptive partition-based counting. We propose to use Voronoi Diagram-based histogram as the vectorization of EPD, without imposing an explicit smooth point transformation model. Under stated separation and normalization conditions, we establish stability bounds and characterize when the histogram representation preserves Wasserstein-scale variation. We demonstrate the effectiveness of our proposed representation  on real-world datasets which have significant topological features for classification and dimensionality reduction tasks.
\end{abstract}

\section{Introduction}
\label{sect: intro}

Topological Data Analysis (TDA) \citep{wasserman2016topological} aims to capture topological features of point cloud $X$ through Persistent Homology (PH) \citep{edelsbrunner2000topological}. PH analyzes the topological features of $X$ through  a nested sequence of simplicial complexes \citep{salnikov2018simplicial} called filtration. The birth and death of a cycle, e.g., a ring (1-dimensional cycle) or a void (2-dimensional cycle), of the sequence are encoded in \textit{Persistence Diagram} (PD) $D=\{r_i=(b_i,d_i)\in\Omega|1\leq i \leq N(D)\}$, where each $r\in D$ is referred as a topological feature, $\Omega=\{(t_1,t_2)\in\mathbb{R}^2|t_2> t_1\}$ is an open half-plane and $N(D)$ is the number of topological features. PD has been used to address problems in various fields, such as biology \cite{liu2022dowker,meng2020weighted,xia2018multiscale,xia2015multidimensional} and chemistry \cite{lee2017quantifying,townsend2020representation}. Each PD can be utilized as a topological summary of a point cloud and used in various machine learning tasks. The challenge is that PD cannot be directly fed as input into a machine learning algorithm. Vectorization \citep{2017persistenceimage,Bubenik2020ThePL,perslay,chevyrev2018persistence,dong2024persistence,polanco2019adaptive} and kernel methods \citep{carriere2017sliced,kusano2017kernel,le2018persistence,Reininghaus2015ASM}  for PD are developed to meet this challenge. 

A key problem in using PD is its high time complexity. As pointed by \citealp{zomorodian2004computing}, the time complexity to obtain PD in the worst case is $O(m^3)$, where $m$ is the number of simplices in the filtration. When we consider the 1-dimensional PD of point cloud $X$, i.e., rings, we need to consider up to 2-simplices and the time complexity can be $O(|X|^9)$. Several works have attempted to reduce the time cost by subsampling \cite{cao2022approximating,chazal2018density,chazal2015subsampling}. We study a vectorization method for a subsample-induced PD called \textit{Expected Persistence Diagram} (EPD). Since a PD $D=\{r_i=(b_i,d_i)\in\Omega|1\leq i \leq N(D)\}$ can be equivalently represented as a measure 
$\mu=\sum_{i=1}^{N(D)}\delta_{r_i}$, where $r_i$ is a topological feature and $\delta_r$ is the Dirac point mass at $r\in \Omega$. Empirical EPD is defined as the average $\bar{\mu}=\frac{1}{n}\sum_{j=1}^{n}\mu_i$, where each $\mu_i$ is the PD of a sampled point cloud $X_i\subset X$ of the same size, referred as sampled PD for short. \textbf{Equivalently, an EPD is a distribution of topological features supported on the open half plane $\Omega=\{(t_1,t_2)\in\mathbb{R}^2|t_2>t_1\}$ \cite{chazal2018density}, with each topological feature $r=(b,d)\in\Omega$ encoding the birth and death of a cycle in the filtration.}


Current methods \cite{chazal2015subsampling,wu2024estimation} build a representation of EPD by averaging the vectors of  sampled PDs in their core computations, i.e. $\frac{1}{n}\sum_{i=1}^{n}\Psi(\mu_i)$. For the choice of $\Psi$, \citealp{chazal2015subsampling} uses the Persistence Landscape (PL); \citealp{wu2024estimation} extends the choice from PL to Persistence Image (PI) \cite{2017persistenceimage}, Persistence Silhouettes (PS) \cite{chazal2014stochastic} and Persistence Weighted Gaussian Kernels (PWGK) \cite{kusano2016persistence}. In order to vectorize each sampled PD $D$, a predefined continuous point transformation function $f$ \cite{perslay}, whose value reaches its peak at $r\in D$, is used to transform point $r$ in PD into a function. The discrete form of $\textbf{op}(\{f(r)\}_{r\in D})$ is then used as the vectorization of $D$, where $\bf op$ is a  permutation invariant operation. When \textbf{op} is summation, averaging the vectors of sampled PDs is equivalent to averaging the vectors $f(r)$ of each topological feature $r$ in EPD. 

Existing EPD vectorizations make different design choices. PI, PS, and PL provide smooth and stable functional summaries, but require selecting a point transformation and discretization grid. These choices can blur local mass differences or emphasize certain regions of the (birth,death) plane. We 
propose to  treat the empirical EPD as a distribution of topological features, and utilize Vrep, a simple histogram based on Voronoi Diagrams \cite{reem2011geometric}, to  represent an EPD effectively. Vrep makes a different choice: it uses a data-dependent Voronoi partition and records empirical mass in each cell.

The main contributions of this work are: (1) Propose to represent EPD as a distribution of topological features by using histogram based on
    an unsupervised Voronoi Diagram partition mechanism. (2) Provide stability analysis and a conditional Wasserstein (an inherent metric applicable to the space of EPDs \cite{divol2021estimation})-related bound that clarifies when Voronoi histograms can separate well-resolved EPDs. (3) Conduct an empirical evaluation of our proposed method on supervised point cloud classification and unsupervised dimensionality reduction tasks.
    

\section{Background}
\label{sect: back}
We provide the pertinent information about PD and EPD (refer to \citealp{chazal2018density,chazal2021introduction} for details).

\textbf{Persistence Diagram.} Let $g:\mathcal{X}\rightarrow\mathbb{R}^{+}$ denote a function on space $\mathcal{X}$, where $\mathbb{R}^{+}$ stands for positive real numbers, $\mathcal{X}$ is a finite and open subset of Euclidean space. $g$ is computed based on a finite point cloud $X$ sampled from manifold $\mathcal{M}\subset\mathcal{X}$. At scale $\epsilon\geq 0$, the sublevel set $\mathcal{X}_{\epsilon}^g=\left\{x\in\mathcal{X}\mid g(x)\leq\epsilon\right\}$ encodes the topological information in $\mathcal{X}$. For $\gamma\leq\epsilon$, we can have the nested sublevel sets $\mathcal{X}_\gamma^g\subseteq \mathcal{X}_\epsilon^g$. By increasing scale $\epsilon$ from $0$, we obtain a filtration, a nested sequence of topological spaces $\{\mathcal{X}_\epsilon^g\}_{0\leq\epsilon<\infty}$. 

\begin{figure}[h]
  \centering
\includegraphics[width=0.9\textwidth]{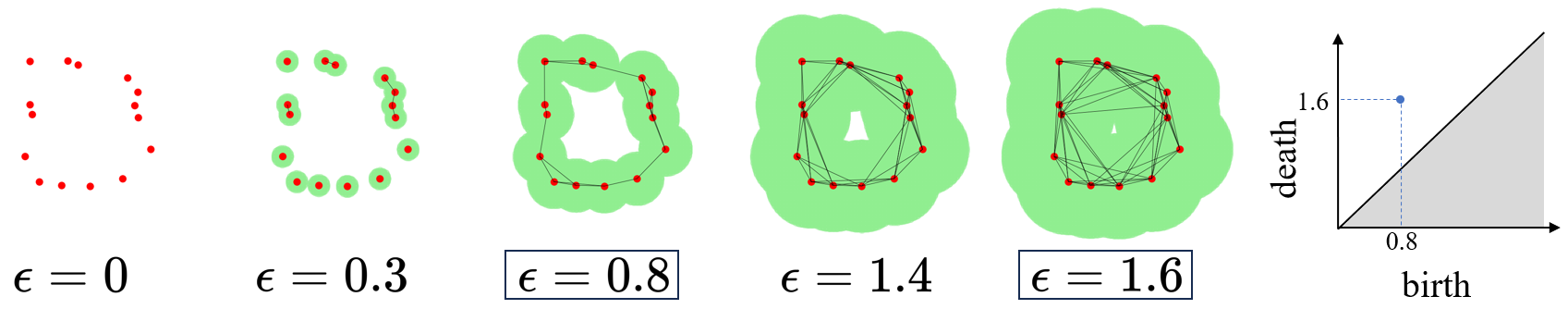}
  \caption{Rips filtration on a 2D point cloud and the corresponding 1-dimensional PD.}
  \label{fig:demo PD}
\end{figure}


A cycle is considered `born' at $b\in\mathbb{R}$ when it first emerges in $\mathcal{X}_{b}^g$, and it `dies' at $d\in\mathbb{R}$ when it ceases to exist in $\mathcal{X}_{p}^g$ for any $p>d$.  0-dimensional cycles are connected components; 1-dimensional cycles are rings or loops; 2-dimensional cycles are voids, etc. Topological feature $r=(b,d)$ is presented in the form of PD $D=\{r_i=(b_i,d_i)\in\Omega|1\leq i \leq N(D)\}$, where $N(D)$ is the number of topological features. In Figure \ref{fig:demo PD}, we choose Rips filtration \cite{vr1995} $g(\cdot)=2\min_{x\in X}\ell(\cdot,x)$, where $\ell$ is Euclidean distance, and there is only one ring in the point cloud which is born at $\epsilon=0.8$ and dies at $\epsilon=1.6$, i.e., $(b,d)=(0.8,1.6)$.

\textbf{Expected Persistence Diagram.} A PD $D=\{r_i=(b_i,d_i)\in\Omega|1\leq i \leq N(D)\}$ can be equivalently represented as a counting measure $\mu$ on $\Omega$ given by $A\in \mathcal{B}\rightarrow\mu(A)=\sum_{i=1}^{N(D)}\delta_{r_i}(A)$,
where $\mathcal{B}$ is the class of all Borel subsets of $\Omega$ and $\delta_r$ denotes the Dirac point mass at $r\in \Omega$. When each sampled PD is a random draw from a distribution $P$, its EPD, denoted as $\mathbb{E}[\mu]$, is defined as $A\in \mathcal{B}\rightarrow  \mathbb{E}[\mu](A) = \mathbb{E}[\mu(A)]$ \cite{chazal2018density}.


\begin{wrapfigure}{r}{0.5\textwidth}
    \vspace{-0.8em}
    \centering
\includegraphics[width=0.24\textwidth]{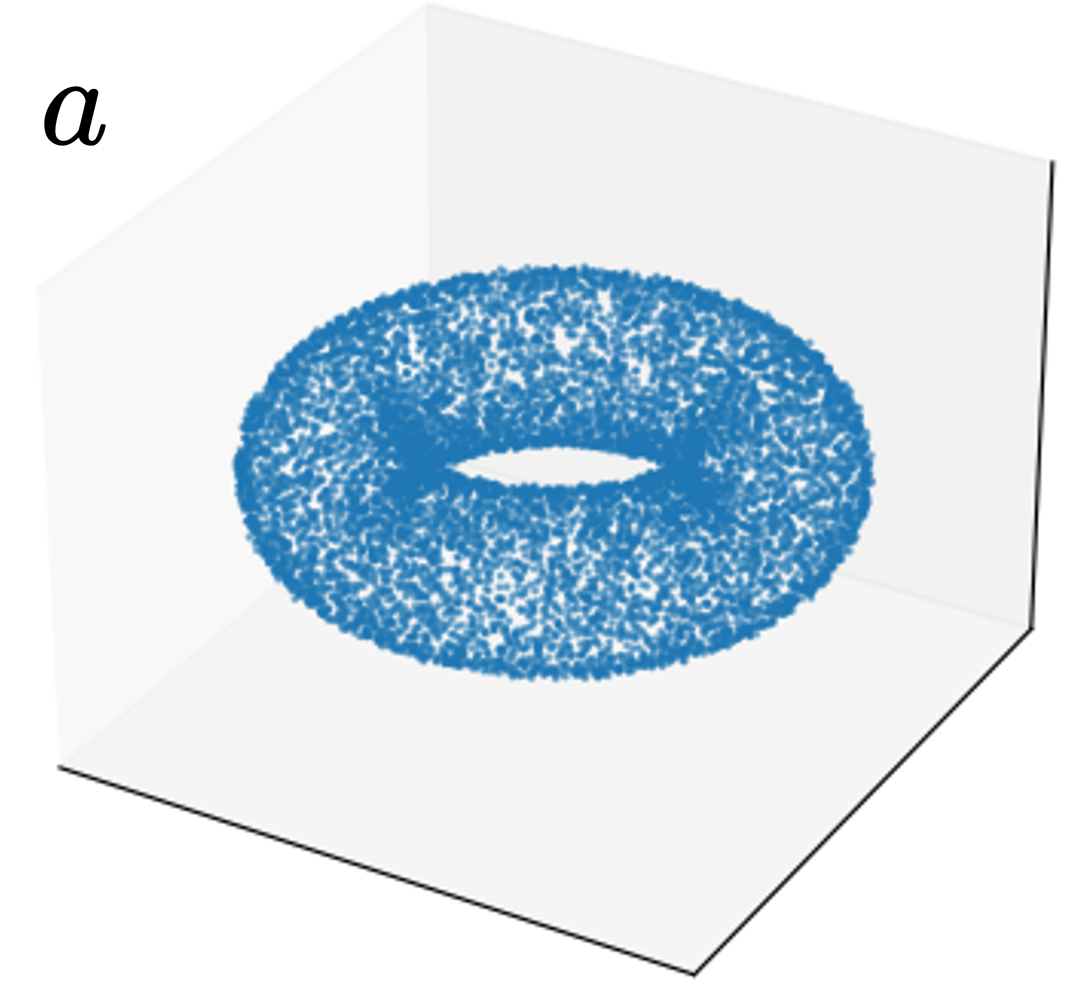}
\includegraphics[width=0.24\textwidth]{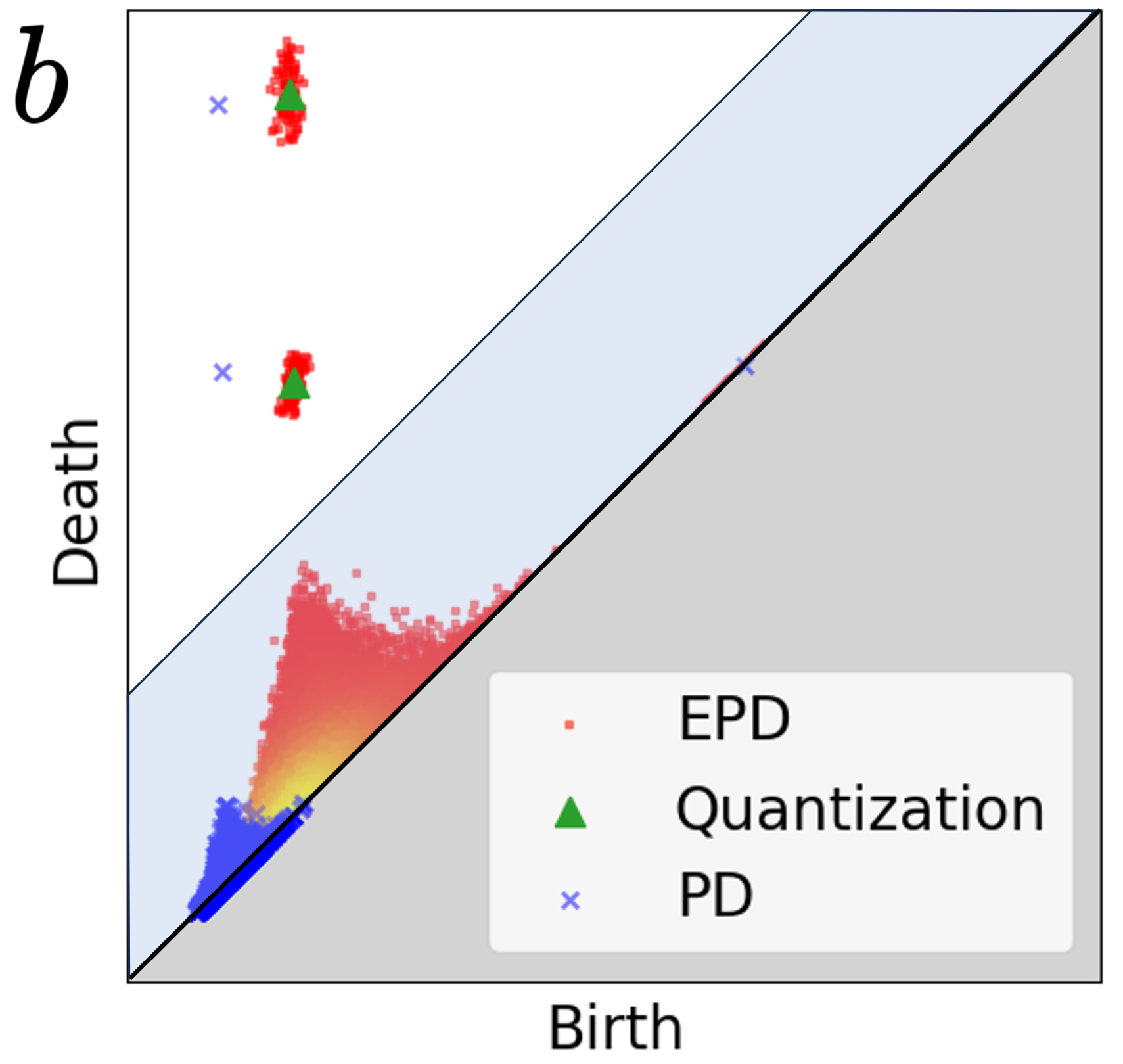}
    \caption{(a) A torus-shaped point cloud with 20000 points, and (b)  its 1-dimensional PD, EPD and EPD Quantization outcomes. The transparent blue area contains noise topological features.}
    \label{fig: EPD_demo}
    \vspace{-0.8em}
\end{wrapfigure}

Given a finite set $\{\mu_1,\mu_2,...,\mu_n\}$, consisting of sampled PDs from $P$, where $n$ is referred as the number of sampled PDs, the empirical EPD is defined as $\bar{\mu}=\frac{1}{n}\sum_{i=1}^{n}\mu_i$. The support of $\bar{\mu}$ is $\mathcal{S}_{\bar{\mu}}=\cup_{i=1}^{n}D_i$, where $D_i=\{r_j=(b_j,d_j)\in\Omega|1\leq j\leq N(D_i)\}$ is the support of the sampled PD $\mu_i$. EPD can be viewed as a distribution \cite{chazal2018density} of topological features supported on the open half plane $\Omega$. Quantization \cite{divol2021estimation} has been developed to reduce the support size of EPD. Examples of the PD, EPD and EPD Quantization outcomes of a point cloud are shown in Figure \ref{fig: EPD_demo}. An inherent metric applicable to the space of EPDs \cite{divol2021estimation} with the same mass, is Wasserstein distance \cite{villani2009wasserstein}. In a space of EPDs with different masses, the Optimal Partial Transport metric ($OT_p$) \cite{figalli2010optimal} is used to allow any mass transportation from or to the diagonal $\partial\Omega$. 

\section{Related Work}
\label{sect: relwork}

Current methods \cite{chazal2015subsampling,wu2024estimation} vectorize EPD by averaging the vectorization\footnote{A detailed description on the vectorization methods of PD is provided in Appendix A.} of each sampled PD. The vectorization of a sampled PD is in the form of $\textbf{op}(\{f(r)\}_{r\in D})$, where \textbf{op} is permutation invariant operation and $f$ is a predefined continuous point transformation function. Current methods can be classified into linear and nonlinear representations according to the choice of \textbf{op}.

 A linear representation \cite{wu2024estimation} $\Psi$ of PD $D = \{r_i = (b_i
,d_i) \in \Omega | 1 \leq i \leq N(D)\}$ is a summary of $D$ in the form $\Psi(D)=\sum_{i=1}^{N(D)} f(r_i)=\int_\Omega f(u) d\mu(u)$,
for a given measurable function $f$ on $\Omega$, where $\mu$ is the measure form of $D$. The corresponding vectorization of EPD is $\mathbb{E}[\Psi(D)]=\int_\Omega f(u)d \mathbb{E}[\mu](u)$. The vectorization of empirical EPD  
is then the average value 
$\frac{1}{n}\sum_{i=1}^{n}\Psi(D_i)$ 
of the vectorized forms of its sampled PDs. According to \citealp{wu2024estimation}, Persistence Image (PI) \cite{2017persistenceimage}, Persistence Silhouettes (PS) \cite{chazal2014stochastic} and Persistence Weighted Gaussian Kernel's feature map (PWGK) \cite{kusano2016persistence} are linear representations.

An existing nonlinear technique is Persistence Landscape (PL) \cite{chazal2015subsampling}. The nonlinearity roots from the kmax operation in $\Psi$ instead of summation, i.e., $\Psi(D)=\text{kmax}_{r\in D}f(r)$, where $f(r)$ is a function defined on closed interval $[0,T]$, for $t\in [0,T]$, $\Psi(D)(t)=\text{kmax}(\{f(r_1)(t),...,f(r_{N})(t)\})$ and kmax is the operation that takes the k-th largest value in the set. This kmax operation can be viewed as an approximation of weighted
summation. The representation of EPD is then the average PL of its sampled PDs.

These transformations introduce smoothness, bandwidth, and grid-resolution choices. Such choices are useful for stability and regularization, but may blur local mass differences or emphasize regions that are not optimal for a given dataset.

\section{Representation of EPD}
\label{sect: Vrep}

We propose to directly treat EPD as a distribution of topological features and develop a representation of EPD based on Voronoi Diagrams built upon a set of random codebooks. A codebook $C$ is a finite set whose element is a point in $\Omega$ or a subset of $\Omega$. Our method is motivated by Lemma 2 in \citealp{divol2021estimation}\footnote{\citealp{divol2021estimation} tries to reduce the size of measure $\mu$ by finding the codebook $C$ that minimize $OT_p(\hat{\mu}(C),\mu)$ and then uses the optimal $C$ as the new support. The relation between our work and \citealp{divol2021estimation} is discussed in Appendix \ref{append: relation}.}, which states that for a measure $\mu$, given any codebook $C= \{c_1,...,c_k\}$, the coefficient $\mu(V(c_i))$ of a Voronoi cell $V(c_i)$ centered at $c_i$ can well represent measure $\mu$. Therefore, we will use this coefficient as the representation.

\textbf{Lemma 2.} [\citealp{divol2021estimation}] Let $C= \{c_1,...,c_k\}$. Let $\hat{\mu}(C)=\sum_{i=1}^{k}\mu(V(c_i))\delta_{c_i}$, where $V(c_i)$ is the Voronoi Cell that centered at $c_i$. Let $\nu=\sum_{j=1}^{k} \alpha_j \delta_{c_j}$ for some coefficients $\alpha_1,...,\alpha_k\geq 0$. Then $OT_p(\hat{\mu}(C),\mu)\leq OT_p(\nu,\mu)$, where $OT_p$ is Optimal Partial Transport metric.






\subsection{Voronoi Diagram-based Representation}
\label{sec-VD}
 For a codebook $C\subset\Omega$, we treat the EPD $\bar{\mu}=\frac{1}{n}\sum_{i=1}^{n}\mu_i$ as a distribution of topological features and define the codebook-based representation of $\bar\mu$ in the form of a histogram \footnote{Similar technique has been used as part of a PD representation named Persistence Bag of Word \cite{zielinski2019persistence}, where the assignment of topological features to each bin is based on Gaussian Mixture Model \cite{reynolds2009gaussian}.}: 
\begin{equation*}
    \Phi(\bar{\mu},C) =[\bar{\mu}(V(c_1)),\bar{\mu}(V(c_2)),...,\bar{\mu}(V(c_k))],
\end{equation*}
where $k$ is the size of codebook $C$ and $V(c_i)$ represents the Voronoi Cell centered at $c_i$, i.e., $V(c_i)=\{x\in\Omega|\forall j\neq i,\|x-c_i\|_2\leq\|x-c_j\|_2\}$.

\begin{wrapfigure}{r}{0.6\textwidth}
    \vspace{-0.8em}
    \centering
\includegraphics[width=0.59\textwidth]{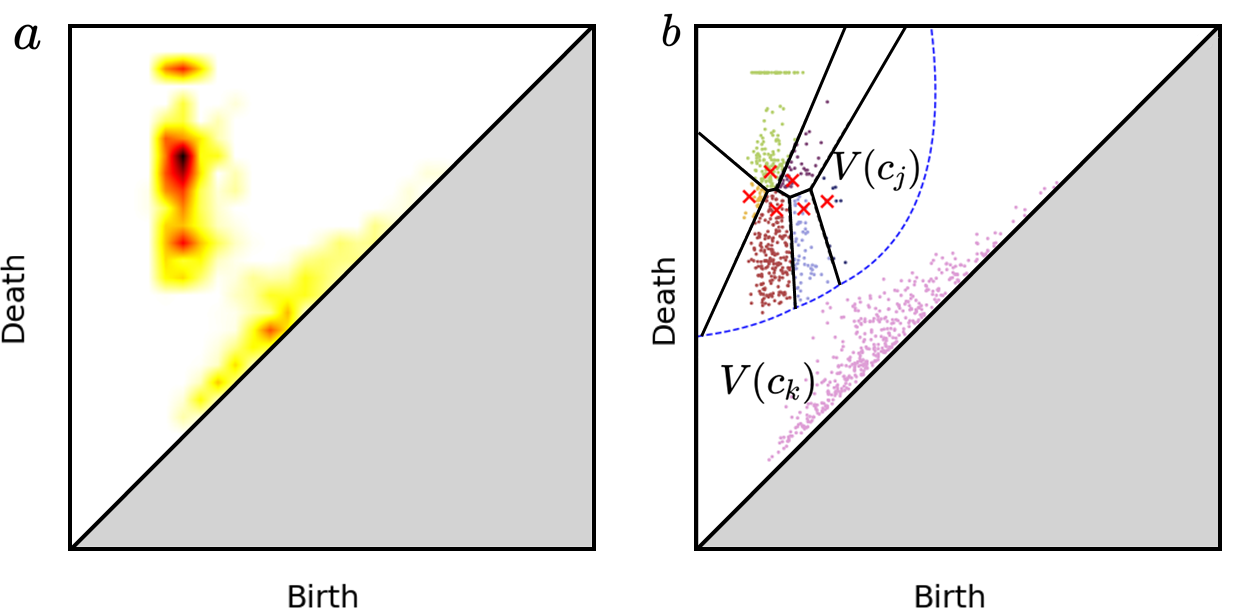}
    \caption{Example of EPD (a) and Voronoi Diagram built from codebook (b). The red crossed points are codebook $C$. The solid black lines are the boundary of the Voronoi Diagram built from $C$. The blue dashed curve is the boundary between the Voronoi Cell due to the diagonal and other Voronoi Cells.}
    \label{fig: vrep}
    \vspace{-1.2em}
\end{wrapfigure}
Although $\Phi(\bar{\mu},C)$ is not designed for each sampled PD, it still can be equivalently expressed in an averaged form, that is, $\Phi(\bar{\mu},C)=\frac{1}{n}\sum_{i=1}^n\Psi(D_i)$, where $\Psi(D_i)=[\mu_i(V(c_1)),...,\mu_i(V(c_k))]$ and $\mu_i$ is the corresponding measure form of $D_i$. In order to align with the setting in \citealp{divol2021estimation}, we define a variant of $\Phi$ which has a codebook of size $k$, with one Voronoi Cell centered at the diagonal $\partial\Omega$, i.e. boundary of the open half plane $\Omega=\{(t_1,t_2)\in\mathbb{R}^2|t_2> t_1\}$. Voronoi Cell centered at the diagonal is defined as $V(c_k)=\{x\in\Omega|\forall j\neq k,\|x-c_j\|_2\geq\min_{y\in c_k}\|x-y\|_2\}$, where $c_k=\partial\Omega$ and each of the other $c_i$ for $i<k$ is a point located inside $\Omega$. An illustration of Voronoi Diagram with a Voronoi Cell center at diagonal is given in Figure \ref{fig: vrep}.



A normalization of $\Phi(\bar{\mu},C)$ to unit $l_1$-norm can be achieved by a preprocessing step which normalizes the EPD, i.e., $\bar{\mu}\leftarrow\bar{\mu}/\bar{\mu}(\Omega)$. We can treat this normalized measure as a distribution since its integral over $\Omega$ is 1. In addition, this normalization ensures that any two EPDs have the same mass. Hence it allows us to measure their dissimilarity by Wasserstein distance directly,
without using the Optimal Partial Transport metric. 
So in order to simplify the analysis, we will consider the representation of the normalized EPD in the rest of this paper. A more detailed discussion on the advantage of normalization is given in Appendix \ref{appendix: norm}.

Given a dataset of EPDs $\{\bar{\mu}_1,\bar{\mu}_2,...,\bar{\mu}_m\}$,  we sample $t$ codebooks $C_i^j\subset\mathcal{S}_{\bar{\mu}_i},  j=1...t$. Each codebook $C_i^j$ is a set of size $k$, with each element being sampled from EPD (distribution of topological features) $\bar{\mu}_i$. This produces a random codebook set $S_i = \{C_i^j\}_{1\leq j \leq t}$ for each EPD $\bar{\mu}_i$.
The final codebook set $S=\cup_{i=1}^{m}S_i$ is used to represent the given dataset of EPDs. From the codebook-based representation $\Phi$, we define the Voronoi-based Representation (Vrep) of EPD as following. 




\begin{definition}[Voronoi-based Representation]
\label{def: Vrep}
Given a codebook set $S=\cup_{i=1}^{m}S_i$ built from an EPD dataset $\{\bar{\mu}_1,\bar{\mu}_2,...,\bar{\mu}_m\}$. The Voronoi-based Representation (Vrep) of an EPD  $\bar{\mu}_i$ is defined as $\hat{\Phi}(\bar{\mu}_i) = \bigoplus_{C\in S} \Phi(\bar{\mu}_i,C),$ where $\bigoplus$ stands for a concatenation operation. An illustration is given in Figure \ref{fig: vrep demo}. 

\end{definition}


\begin{figure}[h]
    \centering
\includegraphics[width=0.6\textwidth]{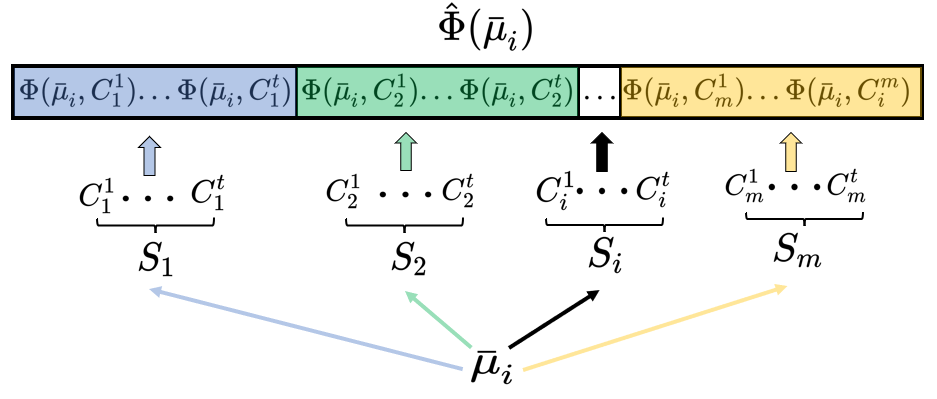}
    \caption{Illustration of Vrep $\hat{\Phi}(\bar{\mu}_i)$ of EPD $\bar{\mu}_i$. Each $S_i$ consists of $t$ codebooks, with each codebook $C$ sampled from EPD $\bar{\mu}_i$.}
    \label{fig: vrep demo}
\end{figure}

As an alternative to Vrep, if each codebook $C\in S$ has  diagonal $c_k=\partial\Omega$  while the other $k-1$ points are in the open half plane $\Omega$, we denote it as $\text{Vrep}_{d}$.


Vrep is data-dependent because for an EPD $\bar{\mu}_i$, $\hat{\Phi}(\bar{\mu}_i)$ is influenced by both the codebook set $S_i$ from EPD $\bar{\mu}_i$ and the codebook set $\cup_{j\neq i}S_j$ from other EPDs in the dataset. We will elaborate this dependence in the Appendix \ref{sect: data-dependence}. 

There are three possible choices of codebook for sampling a codebook $C=\{c_1,...,c_k\}$ from an EPD $\bar{\mu}$ with support $\mathcal{S}_{\bar{\mu}}=\cup_{i=1}^{n}D_i$, where $D_i=\{r_j=(b_j,d_j)\in\Omega|1\leq j \leq N(D)\}$: (1) Default: each $c_i$ is sampled from $\mathcal{S}_{\bar{\mu}}$ with each $r\in \mathcal{S}_{\bar{\mu}}$ having the same weight; (2) Persistence-Weighted: each $c_i$ is sampled from $\mathcal{S}_{\bar{\mu}}$ with each $r=(b,d)\in \mathcal{S}_{\bar{\mu}}$ having weight $d-b$; (3) Uniform: each $c_i$ is sampled from a uniform distribution supported on a fixed rectangle above the diagonal $\partial\Omega$.
 The former two choices are relevant to EPD and the last one is not. For Vrep$_d$, these three choices determine the sampling method for the remaining $k-1$ points in codebook $C$, while $c_k$ is fixed as the diagonal $\partial\Omega$. Unless otherwise stated, the default choice of codebook  is utilized in the proposed EPD representation. We examine the effects of the other two codebook choices in Experiments Section.

Like all finite-dimensional vectorizations of PD/EPD, Vrep is lossy. Its information loss is controlled by the codebook size, Voronoi cell geometry, and normalization. Compared with PI/PS, Vrep does not smooth each point by a fixed analytic kernel, but instead aggregates mass inside data-dependent cells. This makes it potentially more faithful to empirical EPD mass allocation. ATOL \cite{royer2021atol} is another unsupervised PD vectorization method based on codebook; Appendix \ref{append: relation} clarifies how Vrep differs as an EPD-specific Voronoi histogram representation.




\subsection{Stability}
\label{sect: stability}

We show that Vrep is stable under the perturbation of EPD. We start with the definition of EPD perturbation. For an EPD $\bar{\mu}$, a perturbation refers to a minor movement of location of the Dirac point mass of its sampled PDs. We first provide the stability result of the codebook-based representation $\Phi$ and then extend it to Vrep $\hat{\Phi}$.

Let $\mu_i=\sum_{j=1}^{N(D_i)} \delta_{r_j}$ be a  measure of sampled PD $D_i=\{r_j=(b_j,d_j)\in\Omega|1\leq j\leq N(D_i)\}$.

\begin{definition}[EPD Perturbation]
\label{def: EPD perturbation}
    For an EPD $\bar{\mu}=\lim_{n\rightarrow\infty}\frac{1}{n}\sum_{i=1}^{n}\mu_i$, a $\Delta$-perturbation of EPD is defined as the perturbation of $r_j$ in $D_i$ (the support of $\mu_i$), and the perturbation of each $r_j$ is bounded within a $\Delta$-radius open ball, i.e., $\|\hat{r}_j-r_j\|_2<\Delta$, where $\hat{r}_j$ is the perturbed version of $r_j$. And we assume that the random perturbation $\Delta\sim \mathcal{N}(0,\Sigma)$. The perturbed EPD is denoted as $\bar{\mu}^\prime$. 
\end{definition}



This definition implies that $W_1(\bar{\mu},\bar{\mu}^\prime)\leq\Delta$. Given that $\Phi(\bar{\mu}^\prime,C^\prime)-\Phi(\bar{\mu},C)=\Phi(\bar{\mu}^\prime,C^\prime)-\Phi(\bar{\mu},C^\prime)+\Phi(\bar{\mu},C^\prime)-\Phi(\bar{\mu},C)$, there are two kinds of error in the stability analysis we need to consider,i.e., Measure Error \& Codebook Error, as shown in the following inequality: $\|\Phi(\bar{\mu}^\prime,C^\prime)-\Phi(\bar{\mu},C) \|_1\leq \underbrace{\|\Phi(\bar{\mu}^\prime,C^\prime)-\Phi(\bar{\mu},C^\prime)\|_1}_{\text{Measure Error}}+\underbrace{\|\Phi(\bar{\mu},C^\prime)-\Phi(\bar{\mu},C)\|_1}_{\text{Codebook Error}}.$

The Measure Error term is caused by the perturbation of measure $\bar{\mu}$ (with fixed codebook $C^\prime$). The Codebook Error term is caused by the perturbation of codebook $C$ (with fixed measure $\bar{\mu}$). We consider the Measure and Codebook Error in Lemma \ref{lemma: lipschitz} and \ref{lem: stab2} respectively.


\begin{lemma}[Stability w.r.t. perturbation of measure $\bar{\mu}$ with fixed codebook $C$]
\label{lemma: lipschitz}
    For an EPD $\bar\mu=\lim_{n\rightarrow\infty}\frac{1}{n}\sum_{i=1}^{n}\mu_i$ and an optimal matching $\eta$: $\text{support}(\bar{\mu})\rightarrow\text{support}(\bar{\mu}^\prime)$ in the definition of 1-Wasserstein distance between $\bar{\mu}$ and $\bar{\mu}^\prime$ ,given any codebook $C$ of size $k$, it holds that $        \|\Phi(\bar{\mu},C)-\Phi(\bar{\mu}^\prime,C)\|_1\leq M\cdot W_1(\bar{\mu},\bar{\mu}^\prime),$
    where $M$ is a constant determined by measure $\bar{\mu}$.
\end{lemma}
From Definition \ref{def: EPD perturbation}, we have the result on the Measure Error term: $\|\Phi(\bar{\mu},C)-\Phi(\bar{\mu}^\prime,C)\|_1\leq M\Delta$.

We consider the Codebook  Error term in Lemma \ref{lem: stab2}, where the upperbound is related to the dimension $k$ of representation $\Phi$. Perturbation causes a change in a single dimension up to $C_0 \epsilon$. A representation with higher dimension is less stable.
\begin{lemma}[Stability w.r.t. codebook $C$ with fixed measure $\bar{\mu}$]
    \label{lem: stab2}
     For a given codebook $C$ of size $k$, its perturbation version $C^\prime$, and EPD $\bar\mu=\lim_{n\rightarrow\infty}\frac{1}{n}\sum_{i=1}^{n}\mu_i$, it holds that $         \|\Phi(\bar{\mu},C)-\Phi(\bar{\mu},C^\prime)\|_1\leq kC_0 \Delta,$ where $C_0$ is a constant determined by $\bar{\mu}$ and $C$.
\end{lemma}


With Lemma \ref{lemma: lipschitz} and \ref{lem: stab2}, we reach the following Theorem \ref{theorem: stab} about the stability  of Vrep $\hat{\Phi}$, which shows that when the perturbation $\Delta$ is rather small, the total error is Lipschitz continuous w.r.t. $\Delta$.
\begin{theorem}
    \label{theorem: stab}
    For an EPD dataset $\{\bar{\mu}_1,...,\bar{\mu}_m\}$ and $\Delta$-perturbation of every EPD $\bar{\mu}_i$, it holds that $\|\hat{\Phi}(\bar{\mu}_i)-\hat{\Phi}(\bar{\mu}_i^\prime)\|_1\leq L\Delta,$ where $L$ is a constant determined by $m$, $t$, $k$, EPD dataset $\{\bar{\mu}_1,...,\bar{\mu}_m\}$ and codebook set $S$: $L=mt(M^{\max}+kC_0^{\max})$, $C_0^{\max}=\max_{C\in S,j\in[m]}C_0(C,\bar{\mu}_j)$ and  $M^{\max}=\max_{j\in[m]}M(\bar{\mu}_j)$.
\end{theorem}

We demonstrate the stability of Vrep by showing the change in representation under different levels of perturbation. The change in representation is defined as $\delta_{\bar{\mu}} = \|\hat{\Phi}(\bar{\mu}^\prime)-\hat{\Phi}(\bar{\mu})\|_1 / \|\hat{\Phi}(\bar{\mu})\|_1$.

\begin{wrapfigure}{r}{0.42\textwidth}
    \vspace{-0.8em}
    \centering
    \includegraphics[width=0.41\textwidth]{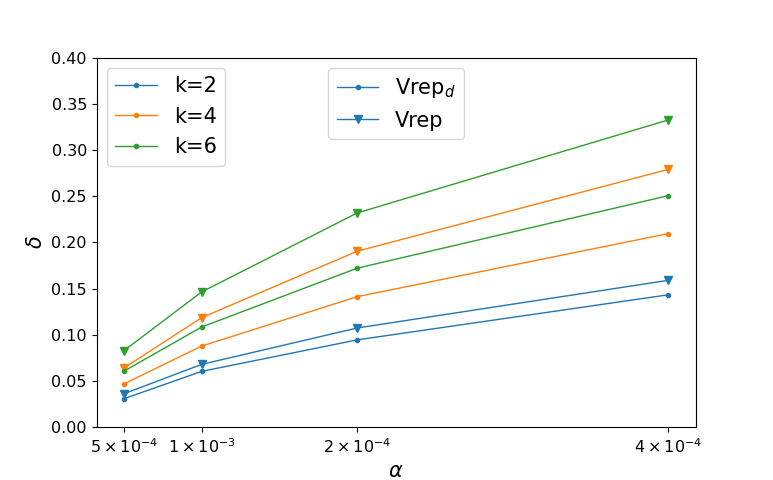}
    \caption{Average change $\delta$ in Vrep and Vrep$_d$ under different EPD perturbation levels $\alpha$.}
    \label{fig:stab}
    \vspace{-0.8em}
\end{wrapfigure}
We use a 3d dynamical system dataset \cite{dong2024persistence,lindstrom2002dynamics}, which describes a discrete food chain model. This dataset contains 9 classes, each class contains 50 point clouds with each point cloud having 2000 points. Each class corresponds to a parameter of the dynamical system. With the dataset of 450 point clouds, for each point cloud $X_i$, we obtain the EPD $\bar{\mu}_i$ by computing 50 sampled PDs from a subset of $X_i$ with size 200. We report the average change of representation $\delta=\frac{1}{450}\sum_{i=1}^{450}\delta_{\bar{\mu}_i}$ under different perturbation levels, where the level $\alpha$ stands for the bandwidth of the Gaussian noise\footnote{Here the noises are from $\mathcal{N}(0,\alpha^{2}I)$. $\Delta$ is the bound to the norm of noise in Definition \ref{def: EPD perturbation}. For $\epsilon\sim\mathcal{N}(0,\alpha^{2}I)$, $\text{E}(\|\epsilon\|)=\alpha\sqrt{\pi/2}$. So $\alpha$ is intuitively linear to $\Delta$ in an expected manner. The change in representation ($\delta$) in Figure \ref{fig:stab} grows slower when $\alpha$ increases. This is consistent with Theorem \ref{theorem: stab}, i.e., the change can be upper bounded by a linear function of $\Delta$.} applied on EPD. The result shown in Figure \ref{fig:stab} is consistent with our analysis that a larger dimension $k$ indicates a less stable vector in Lemma \ref{lem: stab2} and Theorem \ref{theorem: stab}. In addition, Vrep is shown to be less stable than Vrep$_d$. This is because there are less points near the boundary between Voronoi Cell centered at diagonal $\partial\Omega$ and other Voronoi Cells, i.e., the blue dashed curve shown in  Figure \ref{fig: vrep} (b), which separates the noise topological features from the significant ones. The analysis of the effect of $t$ on $\delta$ is provided in Appendix \ref{appendix: t effect}.


\subsection{Dissimilarity of Two EPDs partially controlled by Wasserstein distance under codebook approximation conditions}
\label{sect: WD}
Here we consider two EPDs ($\bar{\mu},\bar{\nu}$) and demonstrate that the most basic component of Vrep (and Vrep$_d$): codebook-based representation $\Phi$ can admit a Wasserstein-related lower bound under approximation conditions.


\begin{theorem}
    \label{thm: sep}
    For any given codebook $C$ with size $k$ and two EPDs ($\bar{\mu},\bar{\nu}$) with the same mass, it holds that $    \|\Phi(\bar{\mu},C)-\Phi(\bar{\nu},C)\|_1 \geq  [W_1(\bar{\mu},\bar{\nu})-W_1(\bar{\mu},\hat{\mu}(c))-W_1(\bar{\nu},\hat{\nu}(c))]/d_{max}(C)$, where $d_{max}(C)=\max_{c_i,c_j\in C}\|c_i-c_j\|_2$ is the diameter of $C$, $W_1$ is  Wasserstein distance \footnote{For two measures $\mu,\nu$ having the same total mass on metric space $(M,\rho)$, Wasserstein distance $W_{p,\rho}$ defined as the infimum of $(\int_{M^2}\rho(x,y)^p)^\frac{1}{p}$ over
all transport plans $\pi$ between $\mu$ and $\nu$, i.e. measures
on $M\times M$ which have for first (resp. second) marginal
$\mu$ (resp. $\nu$ ). When $\rho$ is the Euclidean distance, we write
$W_p$ instead of $W_{p,\rho}$.}, $\hat{\mu}(C)=\sum_{i=1}^{k}\bar{\mu}(V(c_i))\delta_{c_i}$, and $\hat{\nu}(C)=\sum_{i=1}^{k}\bar{\nu}(V(c_i))\delta_{c_i}$.
\end{theorem}

The bound is informative only when the Wasserstein separation is large relative to the codebook approximation errors. Therefore, this result should be interpreted as a conditional consistency statement for well-resolved codebooks, not as a general guarantee that Vrep distances monotonically preserve Wasserstein distances\footnote{When $W_1(\bar{\mu},\bar{\nu})$ is small, the upper bound in Inequality \ref{equation: upper and lower bound} would be more useful, indicating that a smaller $W_1(\bar{\mu},\bar{\nu})$ leads to a smaller $\|\Phi(\bar{\mu},C)-\Phi(\bar{\nu},C)\|_1$.}. Appendix \ref{appendix: wass} provides a controlled synthetic study that directly compares representation distances with Wasserstein distances. The results show that Vrep/Vrep$_d$ are effective when variation is dominated by coarse mass displacement, while smooth representations can be preferable for fine within-cell shifts, supporting the trade-off interpretation of our method. A demo analysis on the Wasserstein separation is given in Appendix \ref{appendix: example distribution}.

\textbf{These results do not imply that Vrep dominates PI, PS, or other vectorizations. They show that, for a fixed codebook satisfying separation conditions, the Voronoi histogram changes Lipschitz-continuously under perturbations of the normalized EPD and can preserve certain distributional differences at the cell level.}

\subsection{Time Complexity}
\label{sect: time}

For an EPD dataset of size $m$, $\{\bar{\mu}_1,...,\bar{\mu}_m\}$, the size of codebook set $S$ is $mtk$, where $k$ is the codebook size and $t$ is the number of codebooks sampled from each EPD. For a particular EPD $\bar{\mu}_i$, the Euclidean distance between each point in $\mathcal{S}_{\bar{\mu}_i}$ and each point in codebook set $S$ needs to be computed for the construction of Voronoi Diagram. So for an EPD $\bar{\mu}$, the time complexity of building Vrep is $O(|\mathcal{S}_{\bar{\mu}}|mtk)$, where $\mathcal{S}_{\bar{\mu}}$ is the support of $\bar{\mu}$.

\textbf{Speed up by subsampling on $\mathcal{S}_{\bar{\mu}}$.} Since we consider $\bar{\mu}$ as a distribution, we can reduce its support size to speed up the process of building a vector representation.  Quantization \cite{divol2021estimation} or simple subsampling can be utilized to reduce the support size. It is more efficient to use subsampling  because Quantization involves an expensive optimization technique. We use subsampling to reduce $\mathcal{S}_{\bar{\mu}}$ to a fixed size and the resultant time complexity becomes $O(mtk)$ for building Vrep for one EPD. The time complexity of building Vrep's for the entire dataset is $O(m^2tk)$. We can further reduce time complexity to $O(mtk)$ by using a random subset of fixed size of codebook set $S$.

\section{Experiments}
\label{sect: exp}

 We compare Vrep with existing representation methods of EPD: Average Persistence Image (PI), Persistence Silhouettes (PS) and Persistence Landscape (PL) in point cloud classification task. We calculate the EPD for each point cloud, convert it into a vector, and then feed it into a Random Forest classifier \cite{breiman2001random}.  We follow the setting in \cite{royer2021atol} and use Random Forests as a ready-to-use tool. Comparable performances can be obtained from using a neural network classifier or other classifiers, depending on the problem. This is a light choice that requires no particular infrastructure or tuning efforts that would produce overly design-dependent results. In addition, we use PointNet \cite{qi2017pointnet}, an end-to-end network for point cloud input, to measure the difficulty of the classification task. Note that PointNet is not a fair contender for the compared methods because the feature extraction in PointNet is a supervised method but others are unsupervised. Code is available in the Supplementary Material.

We also compare Vrep with two more conventional vectorization methods, Betti and Euler Curve in Appendix \ref{appendix: Betti and Euler Curve}. For kernel methods, we compare Vrep with Persistence Weighted Gaussian Kernel and Sliced Wasserstein Kernel in Appendix \ref{appendix: PWGK}, where the classifier is Support Vector Machine \cite{suthaharan2016support}.


We do not compare with supervised vectorization methods like PersLay \cite{perslay} since we aim to show that an unsupervised method (Vrep) with easy classifier like Random Forest can have a competitive performance in comparison with end-to-end supervised PointNet, with lower runtime in our reported settings. \textbf{The unsupervised nature also enables Vrep to be useful in unsupervised task like dimensionality reduction, as shown in Appendix \ref{appendix: DR}.}

\textbf{Datasets}. For datasets with significant topological features, we choose Protein \cite{chandonia2022scope,fox2014scope}, CAD \cite{kim2020large} and Time-Delaying embedding \cite{seversky2016time} (point cloud) derived from time series data \cite{8894743}. For the computation of EPD, each sampled PD is obtained through Rips filtration \cite{vr1995} for the Protein and time series data. CAD uses a more efficient alpha filtration \cite{edelsbrunner1993union} given its large number of points in each point cloud. The detailed hyperparameter setting and description of datasets are provided in Appendix \ref{append: exp setting}.

\textbf{Experiment Settings.} We report the
mean classification accuracy of 10 random splits where $80\%$ of the dataset is for training and the rest is for testing. 3-fold cross validation on the training set is
used to select the best hyperparameters for each approach. \textbf{For Vrep and Vrep$_d$, the test set's representation is constructed based on the codebook set $S$ from the training set.} The results are shown in Table \ref{tab: exp rf}, where $\text{CAD}_{0.01}$ and $\text{CAD}_{0.05}$ stands for the CAD dataset with Gaussian noise of bandwidth being 0.01 and 0.05. The detailed settings and summary of datasets are provided in Appendix \ref{append: exp setting}.

Vrep and Vrep$_d$ outperform PI, PS and PL over all the datasets. Vrep often produces higher accuracy than Vrep$_d$. But from the result of CAD, CAD$_{0.01}$ and CAD$_{0.05}$, we demonstrate that Vrep$_d$ is more stable than Vrep. This is consistent with our stability result shown in Figure \ref{fig:stab}. \textbf{The hyperparameter sensitivity results of Vrep and Vrep$_d$ are shown in Appendix \ref{appendix: sens}}. We fix the value of $t$ and report the performance under different $k$s. Vrep (Vrep$_d$) is rather stable w.r.t. the choice of $k$ when the codebook choice is default or persistence-weighted.

\begin{table}[t]
    \centering
            \caption{Accuracy of Vrep, Vrep$_d$, Persistence Image (PI), Persistence Silhouettes (PS), Persistence Landscape (PL). PointNet \cite{qi2017pointnet} is an end-to-end network designed for point cloud classification. It serves as a means to assess the difficulty of classifying point clouds. The default codebook choice is used in Vrep and Vrep$_d$. A direct computation of PD on the CAD dataset is
not feasible, as most point clouds contain too many points. Since the Protein dataset has distance matrix only and PointNet takes point cloud only as input, we use Multidimensional Scaling \cite{borg2007modern} to transform the distance matrix into a 3d point cloud. Boldface is for the highest accuracy on each dataset (excluding PointNet).
}
\label{tab: exp rf}
\vspace{8pt}
    \resizebox{0.94\textwidth}{!}
    {
    \begin{tabular}{|c|c|c|c|c|c|c|c|c||c|}
    \hline
        ~ & \text{Vrep} & $\text{Vrep}_{d}$ & \multicolumn{2}{c|}{PI} & \multicolumn{2}{c|}{PS} & \multicolumn{2}{c||}{PL} & PointNet\\ \cline{2-10}
        & \multicolumn{2}{c|}{EPD}  & PD & EPD  & PD & EPD & PD & EPD & Point Cloud \\ \hline
        Protein & \textbf{0.985}  & \textbf{0.985}  & 0.955  & 0.790  & 0.880 & 0.850 & 0.970 & 0.950 & 0.740\\ 
        CAD & \textbf{0.931}  & 0.900  & $\backslash$ & 0.911  & $\backslash$ & 0.897  & $\backslash$ & 0.901 & 0.950  \\ 
        CAD$_{0.01}$ & \textbf{0.912}  & 0.900  & $\backslash$ & 0.911  & $\backslash$ & 0.881 & $\backslash$ & 0.900 & 0.946\\ 
        CAD$_{0.05}$ & \textbf{0.923}  & 0.904  & $\backslash$ & 0.900 & $\backslash$ & 0.885 & $\backslash$ & 0.900  & 0.950 \\ 
        Beef & \textbf{0.500}  & 0.483  & 0.367 & 0.358  & 0.283 & 0.333 & 0.317 &0.317 & 0.550  \\ 
        BirdChicken & \textbf{0.963}  & \textbf{0.963}  & 0.763  & 0.838  & 0.875 & 0.875 & 0.875 & 0.875 & 1.000 \\ 
        DPTW & \textbf{0.776}  & 0.767  & 0.640  & 0.742  & 0.690 & 0.704 & 0.650 & 0.650 & 0.784   \\ 
        Earthquakes & \textbf{0.825}  & 0.824  & 0.808 & 0.801  & 0.794 & 0.781 & 0.805  & 0.805&0.805  \\ 
        ECG200 & \textbf{0.850}  & 0.823  & 0.685  & 0.808  & 0.728 & 0.795 & 0.725 & 0.725  & 0.847   \\ \hline
    \end{tabular}}
\end{table}

On these topology-sensitive datasets, Vrep and Vrep$_d$ achieve higher mean accuracy than PI, PS, and PL under the Random Forest pipeline. This result can be interpreted as evidence that adaptive mass aggregation can be effective in these settings. \textit{The comparison with PWGK/SWK in Appendix \ref{appendix: PWGK} further shows that smooth kernel methods can be preferable on several datasets, at substantially higher runtime.}

Compared with PointNet, although Vrep produces comparable or lower accuracy, Vrep can be used in unsupervised task like dimensionality reduction, shown in Section \ref{appendix: DR}, while PointNet can not  be applied to unsupervised tasks due to its requirement of label for the end-to-end training. The time cost comparison with PointNet is provided in the Appendix E.5.

Note that \textbf{Vrep is not uniformly best.} In Appendix \ref{appendix: PWGK}, PWGK performs better on several CAD and DPTW settings, suggesting that smooth kernel methods can be preferable when the relevant topological variation is better captured by continuous similarity rather than coarse mass aggregation.

\begin{table*}[h]
    \centering
            \caption{Accuracy of Vrep and Vrep$_d$ under different codebook choices: default, persistence-weighted and uniform. The first two choices are EPD-relevant while the last one is EPD-irrelevant. Boldface is for the highest accuracy on each dataset.}
    \setlength{\tabcolsep}{2mm}
    {\small
    \begin{tabular}{|c|c|c|c|c|c|c|}
    \hline
    & \multicolumn{3}{c|}{Vrep} & \multicolumn{3}{c|}{Vrep$_d$} \\\cline{2-7}
         & default & persistence & uniform & default & persistence & uniform  \\ \hline
        Protein & \textbf{0.985} & 0.920 & 0.780 & \textbf{0.985} & 0.930 & 0.665\\
        CAD & 0.931 & \textbf{0.969} & 0.935 & 0.900 & {0.961} & 0.877\\
        CAD$_{0.01}$ & 0.912 & \textbf{0.946} & 0.904 & 0.900 & {0.935} & 0.900\\
        CAD$_{0.05}$ & 0.923 & \textbf{0.942} & 0.908 &0.904 & {0.931} & 0.896\\
        Beef & 0.500 & \textbf{0.583} & 0.517 & 0.483 & {0.558} & 0.467 \\
        BirdChicken & \textbf{0.963} & 0.938 & 0.925 & \textbf{0.963} & 0.950 & 0.950 \\
        DPTW & \textbf{0.776} & 0.756 & 0.757 & {0.767} & 0.756 & 0.735\\
        Earthquakes & \textbf{0.825} & 0.814 & 0.823 & {0.824} & 0.812 & 0.818\\
            ECG200 & \textbf{0.850}& 0.828 & 0.830 & 0.823 & 0.81 & {0.835}\\\hline
    \end{tabular}}
    \label{tab: ablation study}
\end{table*}

\subsection{Study on the codebook choice}
\label{sect: ablation study}
We conduct a study on the codebook choice for Vrep and Vrep$_d$. We consider the three codebook choices of the default, persistence-weighted and uniform, stated in the Voronoi Diagram-based Representation Section. As shown in Table \ref{tab: ablation study}, the persistence-weighted codebook outperforms the default one on CAD, CAD$_{0.01}$, CAD$_{0.05}$ and Beef, for both Vrep and Vrep$_d$. The persistence-weighted codebook performs better in the noisy case because the codebook is mainly sampled on significant topological features, which reduces the effect of noisy features near the diagonal.


The highest accuracy is obtained through either the default or persistence-weighted codebook which consistently outperforms the EPD-irrelevant uniform codebook. Further experiments comparing using Quantization results \citep{divol2021estimation} and ATOL-style centers \citep{royer2021atol} as codebooks are given in Appendix \ref{appendix: codebook construction}.

\subsection{Scaleup Test}

There are two components in an EPD dataset that affect the computation effort for vectorization: the number of EPDs in the dataset and the number $n$ of sampled PDs in each EPD. 

For the EPD dataset size, as shown in Figure \ref{fig: scaleup3} (a), Vrep (Vrep$_d$) is quadratic to the dataset size because as shown in Definition \ref{def: Vrep}, a larger dataset leads to a larger codebook set $S$. PI, PL and PS all have linear complexity. For the actual time cost, Vrep (Vrep$_d$) is faster than PL and PS, but slower than PI.

\begin{figure}[h!]
\centering
\includegraphics[width=0.4\textwidth]{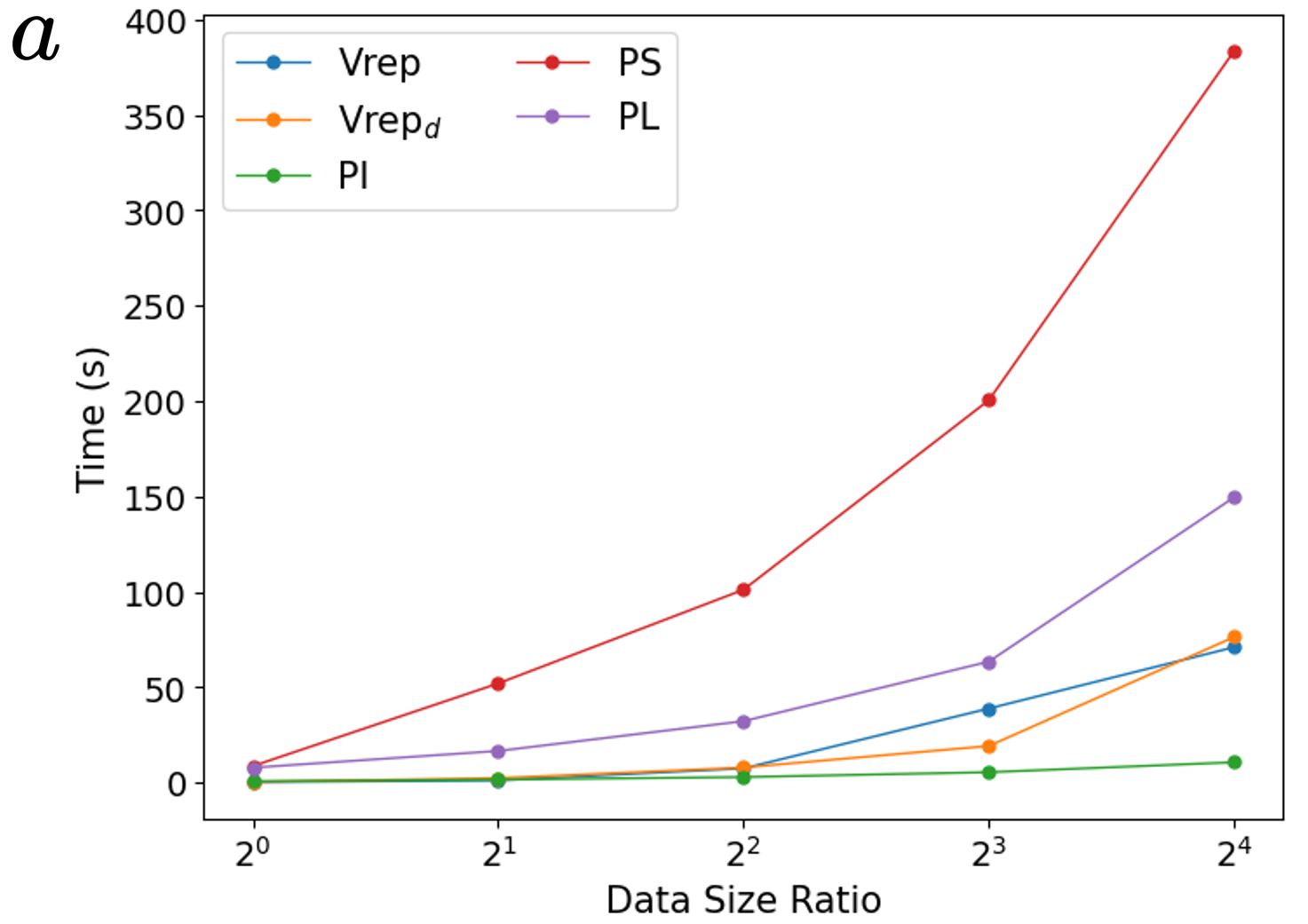}
\includegraphics[width=0.4\textwidth]{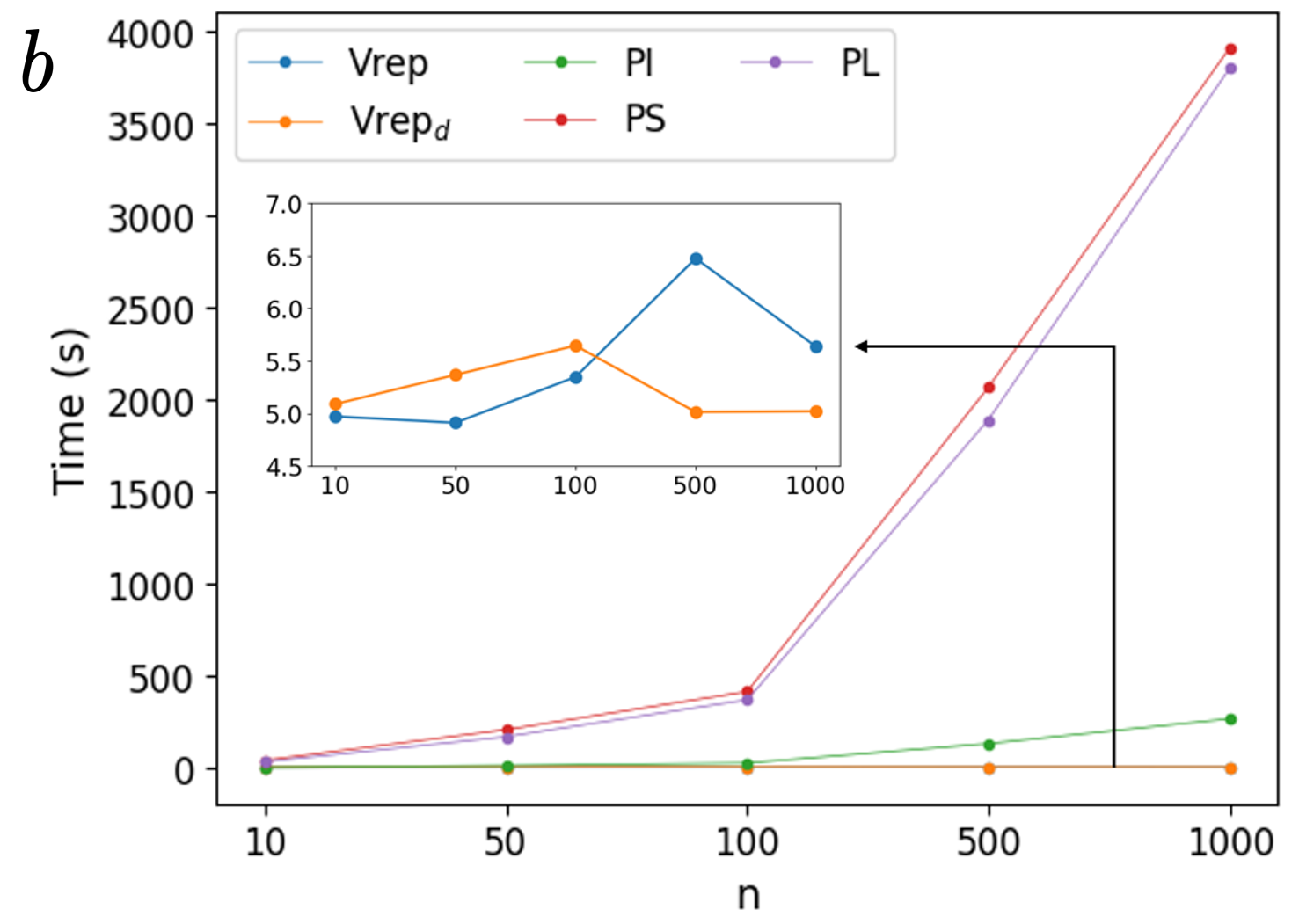}
\caption{(a) Scaleup test on the dataset size, where the dataset size is 126 at data size ratio =1. The number $n$ of sampled PD in each EPD is set to be 10. The actual time cost is reported here. Corresponding time ratio is reported in Appendix E.6. (b) Scaleup test on CAD w.r.t. the number $n$ of sampled PDs in EPD. The overlapping area between Vrep and Vrep$_d$ is enlarged.
}
\label{fig: scaleup3}
\end{figure}

We perform a scale-up test w.r.t. the number $n$ of sampled PDs in each EPD in Figure \ref{fig: scaleup3} (b). The time cost of PI, PS and PL grows as $n$ increases because these methods need to be applied on each sampled PD. Vrep and Vrep$_d$ remain almost unchanged over different values of $n$ because the subsampling technique ensures a fixed size of the EPD support $\mathcal{S}_{\bar{\mu}}$.

\section{Conclusion}

We proposed Voronoi histograms as an adaptive vectorization of Expected Persistence Diagrams. Instead of applying a fixed smooth point transformation to each sampled persistence diagram, Vrep aggregates normalized empirical EPD mass over data-dependent Voronoi cells. This gives a different approximation bias from PI, PS, and PL: Vrep can efficiently summarize coarse mass allocation with respect to the number of sampled PDs, but it remains a lossy finite-dimensional representation.

Our analysis establishes stability under the stated perturbation, normalization, and codebook conditions, and gives a conditional Wasserstein-related bound for well-resolved codebooks. These results should not be interpreted as a general monotonic preservation guarantee or as dominance over existing vectorizations. Empirically, Vrep is competitive on topology-sensitive classification datasets and offers favorable scaling with the number of sampled PDs, while kernel methods such as PWGK and SWK can perform better in some settings at higher computational cost. 

Limitation is discussed in Appendix \ref{appendix: limitation}.





{
\small
\bibliography{epdv}

@article{divol2021estimation,
  title={Estimation and quantization of expected persistence diagrams},
  author={Divol, Vincent and Lacombe, Th{\'e}o},
  journal={International Conference On Machine Learning},
  year={2021}
}

@inproceedings{reem2011geometric,
  title={The geometric stability of Voronoi diagrams with respect to small changes of the sites},
  author={Reem, Daniel},
  booktitle={Proceedings of Annual Symposium on Computational Geometry},
  pages={254--263},
  year={2011}
}

@article{beer1974hausdorff,
  title={The Hausdorff metric and convergence in measure.},
  author={Beer, Gerald A},
  journal={Michigan Mathematical Journal},
  volume={20},
  number={4},
  pages={63--64},
  year={1974}
}

@article{wu2024estimation,
  title={On the estimation of persistence intensity functions and linear representations of persistence diagrams},
  author={Wu, Weichen and Kim, Jisu and Rinaldo, Alessandro},
  journal={International Conference on Artificial Intelligence and Statistics},
  year={2024}
}

@article{van2008visualizing,
  title={Visualizing data using t-SNE.},
  author={Van der Maaten, Laurens and Hinton, Geoffrey},
  journal={Journal of Machine Learning Research},
  volume={9},
  number={11},
  year={2008}
}

@article{wasserman2016topological,
  title={Topological data analysis},
  author={Wasserman, Larry},
  journal={Annual Review of Statistics and Its Application},
  volume={5},
  pages={501--532},
  year={2018},
  publisher={Annual Reviews}
}

@article{Reininghaus2015ASM,
  title={A stable multi-scale kernel for topological machine learning},
  author={Jan Reininghaus and Stefan Huber and Ulrich Bauer and Roland Kwitt},
  journal={IEEE Conference on Computer Vision and Pattern Recognition},
  year={2015}
}

@article{carriere2017sliced,
author = {Carri\`{e}re, Mathieu and Cuturi, Marco and Oudot, Steve},
title = {Sliced Wasserstein Kernel for Persistence Diagrams},
year = {2017},
journal = {International Conference on Machine Learning},
}

@article{kusano2017kernel,
  title={Kernel method for persistence diagrams via kernel embedding and weight factor},
  author={Kusano, Genki and Fukumizu, Kenji and Hiraoka, Yasuaki},
  journal={Journal of Machine Learning Research},
  volume={18},
  number={1},
  pages={6947--6987},
  year={2017},
}

@article{le2018persistence,
  title={Persistence fisher kernel: A riemannian manifold kernel for persistence diagrams},
  author={Le, Tam and Yamada, Makoto},
  journal={Advances in Neural Information Processing Systems},
  year={2018}
}

@article{polanco2019adaptive,
  author={Polanco, Luis and Perea, Jose A.},
  journal={International Conference On Machine Learning And Applications}, 
  title={Adaptive Template Systems: Data-Driven Feature Selection for Learning with Persistence Diagrams}, 
  year={2019}
  }

@article{Bubenik2020ThePL,
  title={The Persistence Landscape and Some of Its Properties},
  author={Peter Bubenik},
  journal={Topological Data Analysis},
  year={2020}
}

@article{kusano2016persistence,
  title={Persistence weighted Gaussian kernel for topological data analysis},
  author={Kusano, Genki and Hiraoka, Yasuaki and Fukumizu, Kenji},
  journal={International Conference on Machine Learning},
  year={2016}
}

@article{meng2020weighted,
  title={Weighted persistent homology for biomolecular data analysis},
  author={Meng, Zhenyu and Anand, D Vijay and Lu, Yunpeng and Wu, Jie and Xia, Kelin},
  journal={Scientific Reports},
  volume={10},
  number={1},
  pages={1--15},
  year={2020},
}

@article{townsend2020representation,
  title={Representation of molecular structures with persistent homology for machine learning applications in chemistry},
  author={Townsend, Jacob and Micucci, Cassie Putman and Hymel, John H and Maroulas, Vasileios and Vogiatzis, Konstantinos D},
  journal={Nature Communications},
  volume={11},
  number={1},
  pages={1--9},
  year={2020},
}

@article{2017persistenceimage,
  title={Persistence images: A stable vector representation of persistent homology},
  author={Adams, Henry and Emerson, Tegan and Kirby, Michael and Neville, Rachel and Peterson, Chris and Shipman, Patrick and Chepushtanova, Sofya and Hanson, Eric and Motta, Francis and Ziegelmeier, Lori},
  journal={Journal of Machine Learning Research},
  volume={18},
  year={2017}
}

@article{hofer2019learning,
  author  = {Christoph D. Hofer and Roland Kwitt and Marc Niethammer},
  title   = {Learning Representations of Persistence Barcodes},
  journal = {Journal of Machine Learning Research},
  year    = {2019},
  volume  = {20},
  number  = {126},
  pages   = {1--45},
}

@inproceedings{zomorodian2004computing,
author = {Zomorodian, Afra and Carlsson, Gunnar},
title = {Computing Persistent Homology},
year = {2004},
booktitle = {Proceedings of Annual Symposium on Computational Geometry},
pages = {347–356},
numpages = {10},
}

@article{vr1995,
  title={On the Vietoris-Rips complexes and a cohomology theory for metric spaces},
  author={Hausmann, Jean-Claude and others},
  journal={Annals of Mathematics Studies},
  volume={138},
  pages={175--188},
  year={1995}
}

@INPROCEEDINGS{edelsbrunner2000topological,
  author={Edelsbrunner, H. and Letscher, D. and Zomorodian, A.},
  booktitle={Proceedings of Annual Symposium on Foundations of Computer Science}, 
  title={Topological persistence and simplification}, 
  year={2000}
  }

@article{bubenik2015statistical,
author = {Bubenik, P.},
year = {2015},
month = {01},
pages = {77-102},
title = {Statistical topological data analysis using persistence landscapes},
volume = {16},
journal = {Journal of Machine Learning Research}
}

@article{chazal2021introduction,
  title={An Introduction to Topological Data Analysis: Fundamental and Practical Aspects for Data Scientists},
  author={Chazal, Fr{\'e}d{\'e}ric and Michel, Bertrand},
  journal={Frontiers in Artificial Intelligence},
  volume={4},
  pages={667963},
  year={2021},
}

@article{liu2022dowker,
Author = {Liu, Xiang and Feng, Huitao and Wu, Jie and Xia, Kelin},
Journal = {PLOS Computational Biology},
Number = {4},
Pages = {1 - 17},
Title = {Dowker complex based machine learning (DCML) models for protein-ligand binding affinity prediction.},
Volume = {18},
Year = {2022},
}

@article{perslay,
  title={Perslay: A neural network layer for persistence diagrams and new graph topological signatures},
  author={Carri{\`e}re, Mathieu and Chazal, Fr{\'e}d{\'e}ric and Ike, Yuichi and Lacombe, Th{\'e}o and Royer, Martin and Umeda, Yuhei},
  journal={International Conference on Artificial Intelligence and Statistics},
  year={2020}
}

@article{chevyrev2018persistence,
  title={Persistence paths and signature features in topological data analysis},
  author={Chevyrev, Ilya and Nanda, Vidit and Oberhauser, Harald},
  journal={IEEE Transactions on Pattern Analysis and Machine Intelligence },
  volume={42},
  number={1},
  pages={192--202},
  year={2018}
}

@article{salnikov2018simplicial,
  title={Simplicial complexes and complex systems},
  author={Salnikov, Vsevolod and Cassese, Daniele and Lambiotte, Renaud},
  journal={European Journal of Physics},
  volume={40},
  number={1},
  pages={014001},
  year={2018},
  publisher={IOP Publishing}
}

@article{xia2018multiscale,
  title={Multiscale persistent functions for biomolecular structure characterization},
  author={Xia, Kelin and Li, Zhiming and Mu, Lin},
  journal={Bulletin of mathematical biology},
  volume={80},
  pages={1--31},
  year={2018},
  publisher={Springer}
}

@article{xia2015multidimensional,
  title={Multidimensional persistence in biomolecular data},
  author={Xia, Kelin and Wei, Guo-Wei},
  journal={Journal of Computational Chemistry},
  volume={36},
  number={20},
  pages={1502--1520},
  year={2015},
  publisher={Wiley Online Library}
}

@article{lee2017quantifying,
  title={Quantifying similarity of pore-geometry in nanoporous materials},
  author={Lee, Yongjin and Barthel, Senja D and D{\l}otko, Pawe{\l} and Moosavi, S Mohamad and Hess, Kathryn and Smit, Berend},
  journal={Nature Communications},
  volume={8},
  number={1},
  pages={1--8},
  year={2017},
  publisher={Nature Publishing Group}
}

@article{cao2022approximating,
  title={Approximating persistent homology for large datasets},
  author={Cao, Yueqi and Monod, Anthea},
  journal={arXiv preprint arXiv:2204.09155},
  year={2022}
}

@article{chazal2015subsampling,
  title={Subsampling methods for persistent homology},
  author={Chazal, Fr{\'e}d{\'e}ric and Fasy, Brittany and Lecci, Fabrizio and Michel, Bertrand and Rinaldo, Alessandro and Wasserman, Larry},
  journal={International Conference on Machine Learning},
  year={2015}
}

@inproceedings{chazal2014stochastic,
  title={Stochastic convergence of persistence landscapes and silhouettes},
  author={Chazal, Fr{\'e}d{\'e}ric and Fasy, Brittany Terese and Lecci, Fabrizio and Rinaldo, Alessandro and Wasserman, Larry},
  booktitle={Proceedings of Annual Symposium on Computational Geometry},
  pages={474--483},
  year={2014}
}

@article{kim2020large,
  title={A Large-Scale Annotated Mechanical Components Benchmark for Classification and Retrieval Tasks with Deep Neural Networks},
  author={Kim, Sangpil and gun Chi, Hyung and Hu, Xiao and Huang, Qixing and Ramani, Karthik},
  journal={European Conference on Computer Vision},
  year={2020}
}

@ARTICLE{8894743,
  author={Dau, Hoang Anh and Bagnall, Anthony and Kamgar, Kaveh and Yeh, Chin-Chia Michael and Zhu, Yan and Gharghabi, Shaghayegh and Ratanamahatana, Chotirat Ann and Keogh, Eamonn},
  journal={IEEE/CAA Journal of Automatica Sinica}, 
  title={The UCR time series archive}, 
  year={2019},
  volume={6},
  number={6},
  pages={1293-1305}}

@article{fox2014scope,
  title={SCOPe: Structural Classification of Proteins—extended, integrating SCOP and ASTRAL data and classification of new structures},
  author={Fox, Naomi K and Brenner, Steven E and Chandonia, John-Marc},
  journal={Nucleic Acids Research},
  volume={42},
  number={D1},
  pages={D304--D309},
  year={2014},
  publisher={Oxford University Press}
}

@article{chandonia2022scope,
  title={SCOPe: improvements to the structural classification of proteins--extended database to facilitate variant interpretation and machine learning},
  author={Chandonia, John-Marc and Guan, Lindsey and Lin, Shiangyi and Yu, Changhua and Fox, Naomi K and Brenner, Steven E},
  journal={Nucleic Acids Research},
  volume={50},
  number={D1},
  pages={D553--D559},
  year={2022},
  publisher={Oxford University Press}
}

@article{nishikawa2024adaptive,
  title={Adaptive topological feature via persistent homology: filtration learning for point clouds},
  author={Nishikawa, Naoki and Ike, Yuichi and Yamanishi, Kenji},
  journal={Advances in Neural Information Processing Systems},
  year={2024}
}

@article{zaheer2017deep,
  title={Deep sets},
  author={Zaheer, Manzil and Kottur, Satwik and Ravanbakhsh, Siamak and Poczos, Barnabas and Salakhutdinov, Russ R and Smola, Alexander J},
  journal={Advances in Neural Information Processing Systems},
  year={2017}
}

@article{hofer2017deep,
  title={Deep learning with topological signatures},
  author={Hofer, Christoph and Kwitt, Roland and Niethammer, Marc and Uhl, Andreas},
  journal={Advances in Neural Information Processing Systems},
  year={2017}
}

@article{hofer2020graph,
  title={Graph filtration learning},
  author={Hofer, Christoph and Graf, Florian and Rieck, Bastian and Niethammer, Marc and Kwitt, Roland},
  journal={International Conference on Machine Learning},
  year={2020}
}

@article{zhang2022gefl,
  title={GEFL: extended filtration learning for graph classification},
  author={Zhang, Simon and Mukherjee, Soham and Dey, Tamal K},
  journal={Learning on Graphs Conference},
  year={2022}
}

@article{horntopological,
  title={Topological Graph Neural Networks},
  author={Horn, Max and De Brouwer, Edward and Moor, Michael and Moreau, Yves and Rieck, Bastian and Borgwardt, Karsten},
  journal={International Conference on Learning Representations},
  year={2021}
}

@article{chintakunta2015entropy,
  title={An entropy-based persistence barcode},
  author={Chintakunta, Harish and Gentimis, Thanos and Gonzalez-Diaz, Rocio and Jimenez, Maria-Jose and Krim, Hamid},
  journal={Pattern Recognition},
  volume={48},
  number={2},
  pages={391--401},
  year={2015},
  publisher={Elsevier}
}

@article{rucco2016characterisation,
  title={Characterisation of the idiotypic immune network through persistent entropy},
  author={Rucco, Matteo and Castiglione, Filippo and Merelli, Emanuela and Pettini, Marco},
  journal={European Conference on Complex Systems},
  year={2016}
}

@article{atienza2020stability,
  title={On the stability of persistent entropy and new summary functions for topological data analysis},
  author={Atienza, Nieves and Gonz{\'a}lez-D{\'\i}az, Roc{\'\i}o and Soriano-Trigueros, Manuel},
  journal={Pattern Recognition},
  volume={107},
  pages={107509},
  year={2020},
  publisher={Elsevier}
}

@article{asaad2022persistent,
  title={Persistent homology for breast tumor classification using mammogram scans},
  author={Asaad, Aras and Ali, Dashti and Majeed, Taban and Rashid, Rasber},
  journal={Mathematics},
  volume={10},
  number={21},
  pages={4039},
  year={2022},
  publisher={MDPI}
}

@article{perea2023approximating,
  title={Approximating continuous functions on persistence diagrams using template functions},
  author={Perea, Jose A and Munch, Elizabeth and Khasawneh, Firas A},
  journal={Foundations of Computational Mathematics},
  volume={23},
  number={4},
  pages={1215--1272},
  year={2023},
  publisher={Springer}
}

@article{kalivsnik2019tropical,
  title={Tropical coordinates on the space of persistence barcodes},
  author={Kali{\v{s}}nik, Sara},
  journal={Foundations of Computational Mathematics},
  volume={19},
  number={1},
  pages={101--129},
  year={2019},
  publisher={Springer}
}

@article{ferri1999representing,
  title={Representing size functions by complex polynomials},
  author={Ferri, Massimo and Landi, Claudia},
  journal={Proc. Math. Met. in Pattern Recognition},
  volume={9},
  pages={16--19},
  year={1999}
}

@incollection{di2015comparing,
  title={Comparing Persistence Diagrams Through Complex Vectors},
  author={Di Fabio, Barbara and Ferri, Massimo},
  booktitle={Image Analysis and Processing},
  pages={294--305},
year = {2015},
publisher={Springer}
}

@article{chung2022persistence,
  title={Persistence curves: A canonical framework for summarizing persistence diagrams},
  author={Chung, Yu-Min and Lawson, Austin},
  journal={Advances in Computational Mathematics},
  volume={48},
  number={1},
  pages={6},
  year={2022},
  publisher={Springer}
}

@article{bubenik2017persistence,
  title={A persistence landscapes toolbox for topological statistics},
  author={Bubenik, Peter and D{\l}otko, Pawe{\l}},
  journal={Journal of Symbolic Computation},
  volume={78},
  pages={91--114},
  year={2017},
  publisher={Elsevier}
}

@article{dong2024persistence,
  title={Persistence B-spline grids: stable vector representation of persistence diagrams based on data fitting},
  author={Dong, Zhetong and Lin, Hongwei and Zhou, Chi and Zhang, Ben and Li, Gengchen},
  journal={Machine Learning},
  volume={113},
  number={3},
  pages={1373--1420},
  year={2024},
  publisher={Springer}
}

@article{villani2009wasserstein,
  title={The wasserstein distances},
  author={Villani, C{\'e}dric},
  journal={Optimal transport: old and new},
  pages={93--111},
  year={2009},
  publisher={Springer}
}

@article{qi2017pointnet,
  title={Pointnet: Deep learning on point sets for 3d classification and segmentation},
  author={Charles, R Qi and Su, Hao and Kaichun, Mo and Guibas, Leonidas J},
  journal={IEEE Conference on Computer Vision and Pattern Recognition},
  year={2017}
}

@inproceedings{chazal2018density,
  title={The Density of Expected Persistence Diagrams and its Kernel Based Estimation},
  author={Chazal, Fr{\'e}d{\'e}ric and Divol, Vincent},
  booktitle={Proceedings of International Symposium on Computational Geometry},
  year={2018}
}

@article{figalli2010optimal,
  title={The optimal partial transport problem},
  author={Figalli, Alessio},
  journal={Archive for Rational Mechanics and Analysis},
  volume={195},
  number={2},
  pages={533--560},
  year={2010},
  publisher={Springer}
}

@article{wang2024multiset,
  title={Multiset Transformer: Advancing Representation Learning in Persistence Diagrams},
  author={Wang, Minghua and Huang, Ziyun and Xu, Jinhui},
  journal={arXiv preprint arXiv:2411.14662},
  year={2024}
}

@article{seversky2016time,
  title={On time-series topological data analysis: New data and opportunities},
  author={Seversky, Lee M and Davis, Shelby and Berger, Matthew},
  journal={IEEE Conference on Computer Vision and Pattern Recognition workshops},
  year={2016}
}

@inproceedings{edelsbrunner1993union,
  title={The union of balls and its dual shape},
  author={Edelsbrunner, Herbert},
  booktitle={Proceedings of Annual Symposium on Computational Geometry},
  pages={218--231},
  year={1993}
}

@article{breiman2001random,
  title={Random forests},
  author={Breiman, Leo},
  journal={Machine Learning},
  volume={45},
  pages={5--32},
  year={2001},
  publisher={Springer}
}

@article{suthaharan2016support,
  title={Support vector machine},
  author={Suthaharan, Shan and Suthaharan, Shan},
  journal={Machine learning models and algorithms for big data classification: thinking with examples for effective learning},
  pages={207--235},
  year={2016},
  publisher={Springer}
}

@book{borg2007modern,
  title={Modern multidimensional scaling: Theory and applications},
  author={Borg, Ingwer and Groenen, Patrick JF},
  year={2007},
  publisher={Springer Science \& Business Media}
}

@article{vaswani2017attention,
  title={Attention is all you need},
  author={Vaswani, A},
  journal={Advances in Neural Information Processing Systems},
  year={2017}
}

@article{lindstrom2002dynamics,
  title={On the dynamics of discrete food chains: Low-and high-frequency behavior and optimality of chaos},
  author={Lindstr{\"o}m, Torsten},
  journal={Journal of Mathematical Biology},
  volume={45},
  number={5},
  pages={396--418},
  year={2002},
  publisher={Springer}
}

@article{calinski1974dendrite,
  title={A dendrite method for cluster analysis},
  author={Cali{\'n}ski, Tadeusz and Harabasz, Jerzy},
  journal={Communications in Statistics-theory and Methods},
  volume={3},
  number={1},
  pages={1--27},
  year={1974},
  publisher={Taylor \& Francis}
}

@article{zielinski2019persistence,
  title={Persistence bag-of-words for topological data analysis},
  author={Zieli{\'n}ski, Bartosz and Lipi{\'n}ski, Micha{\l} and Juda, Mateusz and Zeppelzauer, Matthias and D{\l}otko, Pawe{\l}},
  journal={International Joint Conference on Artificial Intelligence},
  year={2019}
}

@article{reynolds2009gaussian,
  title={Gaussian mixture models.},
  author={Reynolds, Douglas A and others},
  journal={Encyclopedia of biometrics},
  volume={741},
  number={659-663},
  year={2009},
  publisher={Berlin, Springer}
}

@inproceedings{royer2021atol,
  title={Atol: measure vectorization for automatic topologically-oriented learning},
  author={Royer, Martin and Chazal, Fr{\'e}d{\'e}ric and Levrard, Cl{\'e}ment and Umeda, Yuhei and Ike, Yuichi},
  booktitle={International Conference on Artificial Intelligence and Statistics},
  pages={1000--1008},
  year={2021},
  organization={PMLR}
}
\bibliographystyle{plain}
}


\appendix

\section{Vectorization methods of PD}
\label{appendix: pd vect}
Here we introduce unsupervised and supervised vectorization methods for PD.
\begin{itemize}
    \item For unsupervised methods, the simplest category is the basic statistical quantities of PD, including persistence entropy \cite{atienza2020stability,chintakunta2015entropy,rucco2016characterisation}, average and
standard deviation of birth, death and lifespan (death-birth) \cite{asaad2022persistent}, etc. The second category aims to construct polynomial maps (template functions) from PD and evaluate the selected maps on PD to get a vector \cite{di2015comparing,ferri1999representing,kalivsnik2019tropical,perea2023approximating}. The third category is representation based on function that uses PD as its parameter, such as Persistence Curve \cite{chung2022persistence}, Persistence Landscape \cite{bubenik2015statistical,bubenik2017persistence}, Persistence Image \cite{2017persistenceimage}, Persistence Silhouettes \cite{chazal2014stochastic} and Persistence B-spline grids \cite{dong2024persistence}. 

\item A representative supervised method is PersLay \cite{perslay}, a neural network layer for PD based on Deep Sets \cite{zaheer2017deep}, which can be integrated into filtration learning framework for point cloud \cite{nishikawa2024adaptive}. Other methods like topological signature \cite{hofer2017deep,hofer2019learning} are based on learnable exponential functions as structure elements. Similar methods are utilized as part of graph filtration learning framework \cite{hofer2020graph,horntopological,zhang2022gefl}. Recent work \cite{wang2024multiset} focuses on vectorization via Transformer \cite{vaswani2017attention}.
\end{itemize}

\section{Relation with Quantization \& ATOL}
\label{append: relation}

We show the concept difference between Vrep and Quantization/ATOL. The experiments comparing using Quantization results and ATOL-style centers as codebooks are given in Appendix \ref{appendix: codebook construction}.

\textbf{Quantization.} Compared with EPD quantization \cite{divol2021estimation}, which aims to reduce the support size of EPD and use the resultant reduced EPD for vectorization in ML tasks, our work focus on directly vectorization of EPD without the middle quantization step.

\textbf{ATOL.}
ATOL \cite{royer2021atol} is closely related to our work in that both methods view persistence-based objects as finite measures in a Euclidean space and construct unsupervised vector representations that can be used by standard machine learning models. ATOL learns a collection of representative regions from a training set of persistence diagrams and encodes each measure through responses to localized contrast functions around learned centers. In this sense, ATOL is a general measure-vectorization framework for topological descriptors, with an encoding budget controlling the number of learned regions.

Vrep shares the high-level goal of converting measure-valued topological summaries into vectors, but it differs from ATOL in both the object being represented and the form of the representation. ATOL is primarily designed to vectorize individual persistence diagrams or general finite measures. In contrast, Vrep is designed for empirical Expected Persistence Diagrams (EPDs), where each input is already an averaged measure obtained from many sampled persistence diagrams. Thus, Vrep targets the distribution of topological features induced by subsampling, rather than a single persistence diagram.

The representation mechanism is also different. ATOL uses learned centers together with contrast functions, producing feature values that depend on the chosen functional responses around these centers. Vrep instead uses a hard Voronoi partition: each coordinate is the normalized mass of the empirical EPD inside a Voronoi cell. Therefore, Vrep can be interpreted as a direct histogram of the EPD over data-dependent cells. This makes the approximation bias explicit: Vrep preserves cell-level mass allocation but discards within-cell geometry, while ATOL can retain smoother localized responses depending on its contrast functions.

The codebook construction also plays a different role. In ATOL, centers are learned to summarize discriminative regions of a collection of persistence diagrams under a fixed encoding budget. In Vrep, codebooks are sampled from EPD supports and used to define Voronoi cells for measuring empirical EPD mass. The final Vrep representation concatenates histograms over a set of such codebooks. The diagonal-aware variant Vrep$_d$ further includes a Voronoi cell associated with the diagonal, which is useful for separating near-diagonal topological noise from more persistent features.

From this perspective, Vrep can be viewed as an EPD-specific, hard-assignment counterpart to general measure vectorization methods such as ATOL. The two methods make different trade-offs. ATOL provides a flexible learned vectorization for persistence measures through contrast functions and an encoding budget, while Vrep focuses on the empirical distributional structure of EPDs and emphasizes direct mass aggregation, stability under EPD perturbations, and computational scaling with respect to the number of sampled persistence diagrams. Our goal is therefore not to replace ATOL, but to provide a complementary representation tailored to EPDs.

\section{Proofs}
\label{appendix: proof}

\subsection{Proof of Lemma \ref{lemma: lipschitz}}
    Lemma \ref{lemma: lipschitz}. For an EPD $\bar\mu=\lim_{n\rightarrow\infty}\frac{1}{n}\sum_{i=1}^{n}\mu_i$ and an optimal matching $\eta$: $\text{support}(\bar{\mu})\rightarrow\text{support}(\bar{\mu}^\prime)$ in the definition of 1-Wasserstein distance between $\bar{\mu}$ and $\bar{\mu}^\prime$ ,given any codebook $C$ of size $k$, it holds that
    \begin{equation*}
        \|\Phi(\bar{\mu},C)-\Phi(\bar{\mu}^\prime,C)\|_1\leq M\cdot W_1(\bar{\mu},\bar{\mu}^\prime),
    \end{equation*}
    where $M$ is a constant.

\begin{proof}
    Denote the pdf of measure $\bar{\mu}$ as $p$, whose existence is proved in \citealp{chazal2018density} and the pdf of perturbation $\epsilon\sim\mathcal{N}(0,\Sigma)$ as $q$. And $p$ is $k$ times differentiable if the sample space to construct EPD is $k$ times differentiable (Theorem 3.5 in \citealp{chazal2018density}). Denote the corresponding random variable of $p$ as $\alpha$. Given that $\alpha$ and $\epsilon$ is dependent of each other, then the pdf of the perturbed random variable $\alpha+\epsilon$ is the convolution of $p$ and $q$:
    \begin{equation*}
        r(y)=(p\ast q)(y)=\int p(y-z)q(z)dz.
    \end{equation*}

    In practice, we use set the max filtration value to be finite, this means that $p$ is a $C^k$ function in a closed bounded area $\Omega^\prime$ that is a subset of $\Omega$. So $p$ is Lipschitz continuous. Denote the Lipschitz constant as $M$, for $x\in \Omega^\prime$, we have that    
    \begin{equation*}
    \begin{aligned}
|p(x)-r(x)|&=|p(x)-\int q(z)p(x-z) dz|\\
        &=|\int q(z)[p(x-z)-p(x)]dz|\\
        & \leq \int q(z) |p(x-z)-p(x)|dz\\
        & \leq M \int q(z) \|z\|dz
            \end{aligned}
    \end{equation*}

Based on our definition of the perturbation, it holds that $W_1(\bar{\mu},\bar{\mu}^\prime)=\leq \int q(z) \|z\|dz$. So we have 
\begin{equation*}
    |p(x)-r(x)|\leq M\cdot W_1(\bar{\mu},\bar{\mu}^\prime).
\end{equation*}
For any Voronoi Cell $V(c_j)$, we have that
\begin{equation*}
    \begin{aligned}
        |\bar{\mu}(V(c_j))-\bar{\mu}^\prime(V(c_j))|&=|\int_{V(c_j)}p(x)-r(x) dx|\\
        & \leq \int_{V(c_j)} |p(x)-r(x)|dx\\
        & \leq M\cdot W_1(\bar{\mu},\bar{\mu}^\prime)\cdot Vol(V(c_j)).
    \end{aligned}
\end{equation*}

As a result, we have 
\begin{equation*}
\begin{aligned}
        \|\Phi(\bar{\mu},C)-\Phi(\bar{\mu}^\prime,C)\|_1& =\sum_{j=1}^{k}|\bar{\mu}(V(c_j))-\bar{\mu}^\prime(V(c_j))|\\
        &\leq \sum_{j=1}^k M\cdot W_1(\bar{\mu},\bar{\mu}^\prime)\cdot Vol(V(c_j))\\
        & = M\cdot W_1(\bar{\mu},\bar{\mu}^\prime).
\end{aligned}
\end{equation*}


\end{proof}

\subsection{Proof of Lemma \ref{lem: stab2}.}
\textbf{Lemma \ref{lem: stab2}.}               For a given codebook $C$ of size $k$, its perturbation version $C^\prime$, and EPD $\bar\mu=\lim_{n\rightarrow\infty}\frac{1}{n}\sum_{i=1}^{n}\mu_i$, it holds that
     \begin{equation*}
         \|\Phi(\bar{\mu},C)-\Phi(\bar{\mu},C^\prime)\|_1\leq kC_0 \Delta,
     \end{equation*}
     where $C_0$ is a constant determined by $\bar{\mu}$.
\begin{proof}

Similar to Theorem 5.1 \cite{beer1974hausdorff}, we assume that :
\begin{itemize}
    \item $\beta=\min_{j\neq j}\|c_i-c_j\|>0$.
    \item There exists $\rho$ such that $B(x,\rho)\cap C\setminus\{c_i\}\neq \emptyset$ for $x\in\Omega^\prime$, where $\Omega^\prime$ is the bounded and closed space we use in the proof of Lemma \ref{lemma: lipschitz}.
\end{itemize}

The first assumption means that there are no overlapping elements in the codebook $C$. For the second assumption, we can choose $\rho=D$, i.e. the diameter of $\Omega^\prime$, and $B(x,\rho)$ would cover $\Omega^\prime$. So it holds that $\forall i\in[k], B(x,\rho)\cap C\setminus\{c_i\}\neq\emptyset$.

According to Theorem 5.1 and Remark 5.3 in \citealp{beer1974hausdorff}, let $\beta=\min_{i\neq j}\|c_i-c_j\|>0$, we have that $\forall \Delta=O(\epsilon)=\gamma\epsilon$, where $0<\epsilon<\beta/6$, $\epsilon\leq 8\cdot d(\bigcup_{i\in[k]}c_i,\partial\Omega^\prime)$ and $\gamma=\min \{\frac{1}{16}\hat{\delta}(\frac{\beta}{12\rho+5\beta}),\frac{d(\bigcup_{i\in[k]}c_i,\partial\Omega^\prime)}{8(\rho+\beta/6)}\}$\footnote{Here $\hat{\delta}$ is the modulus of convexity, i.e. a function $\hat{\delta}: [0,2]\rightarrow[0,1]$ defined by $\hat{\delta}(\epsilon)=\inf\{1-|(x+y)/2|:\;|x-y|\geq\epsilon,|x|=|y|=1\}$. There are more convent functions for specific cases. The detials are given in page 6 of \citealp{beer1974hausdorff} (\url{https://arxiv.org/pdf/1103.4125}).}, it holds that $d_H(V(c_i),V(c_i^\prime))\leq\epsilon$.

In order to relate the upper bound to $\Delta$, we rewrite the result above as following: for $0<\Delta<\gamma\beta/6$ and $\Delta\leq8\cdot d(\bigcup_{i\in[k]}c_i,\partial\Omega^\prime)\cdot \gamma$, it holds that
\begin{equation*}
    d_H(V(c_i),V(c_j))\leq\frac{\Delta}{\gamma}.
\end{equation*}

$V(c_i),V(c_j)\in\Omega^\prime$ is a bounder area in the plane, we denote the total Perimeter of these areas as $L_i$. Then according to $d_H(A,B)\leq\frac{\Delta}{\gamma}$, it holds that
\begin{equation*}
    \text{Vol}(V(c_i)\triangle V(c_j))\leq L_{\max}\cdot \frac{\Delta}{\gamma},
\end{equation*}
where $L_{\max}=\max_{i\in[k]}L_i$.

Given the fact that $p$ (pdf of $\bar{\mu}$) is a continuous function in the bounded and closed space $\Omega^\prime$, we know that $p$ is bounded and let $M_p=\|p\|_\infty=\text{sup}_{x\in\Omega^\prime}p(x)$. Then it holds that
\begin{equation*}
    |\bar{\mu}(V(c_i)-\bar{\mu})(V(c_i^\prime))|\leq M_p\cdot\text{Vol}(V(c_i)\triangle V(c_j))\leq M_p\cdot L_{\max}\cdot\frac{\Delta}{\gamma}.
\end{equation*}
Then 
\begin{equation*}
    \|\Phi(\bar{\mu},C)-\Phi(\bar{\mu},C^\prime)\|_1=\sum_{i\in[k]}|\bar{\mu}(V(c_i)-\bar{\mu})(V(c_i^\prime))|\leq kC_0\Delta,
\end{equation*}
where $C_0=M_p L_{\max} /\gamma$.

\end{proof}
\subsection{Proof of Theorem \ref{theorem: stab}.}
   \textbf{Theorem \ref{theorem: stab}}. For an EPD dataset $\{\bar{\mu}_1,...,\bar{\mu}_m\}$ and $\Delta$-perturbation of every EPD $\bar{\mu}_i$, it holds that
    \begin{equation*}
        \|\hat{\Phi}(\bar{\mu}_i)-\hat{\Phi}(\bar{\mu}_i^\prime)\|_1\leq L\Delta,
    \end{equation*}
    where $L=mt(M^{\max}+kC_0^{\max})$ is a constant determined by $m$, $t$, $k$, EPD dataset $\{\bar{\mu}_1,...,\bar{\mu}_m\}$ and codebook set $S$.

\begin{proof}
Based on the fact that $\Phi(\bar{\mu}^\prime,C^\prime)-\Phi(\bar{\mu},C)=\Phi(\bar{\mu}^\prime,C^\prime)-\Phi(\bar{\mu},C^\prime)+\Phi(\bar{\mu},C^\prime)-\Phi(\bar{\mu},C)$, it holds that
\begin{equation*}
    \|\Phi(\bar{\mu}^\prime,C^\prime)-\Phi(\bar{\mu},C)\|_1 \leq \|\Phi(\bar{\mu}^\prime,C^\prime)-\Phi(\bar{\mu},C^\prime)\|_1+\|\Phi(\bar{\mu},C^\prime)-\Phi(\bar{\mu},C)\|_1
\end{equation*}
According to Lemma \ref{lem: stab2} \& \ref{lemma: lipschitz}, for $0<\Delta<\gamma\beta/6$ and $\Delta\leq8\cdot d(\bigcup_{i\in[k]}c_i,\partial\Omega^\prime)\cdot \gamma$, it holds that 
\begin{equation*}
    \|\Phi(\bar{\mu}^\prime,C^\prime)-\Phi(\bar{\mu},C)\|_1\leq (M+kC_0)\Delta.
\end{equation*}

Take $M^{\max}=\max_{C\in S,j\in[m]}M(C,\bar{\mu}_j)=\max_{j\in[m]}M(\bar{\mu}_j)$, $C_0^{\max}=\max_{C\in S,j\in[m]}C_0(C,\bar{\mu}_j)$, for $\Delta$ that satisfies $0<\Delta<\min_{C\in S}\gamma(C)\beta(C)/6$ and $\Delta\leq \min_{C\in S}8\cdot d(\bigcup_{i\in[k]}c_i,\partial\Omega^\prime)\cdot \gamma(C)$, it holds that


\begin{equation*}
\begin{aligned}
\|\hat{\Phi}(\bar{\mu}_i)-\hat{\Phi}(\bar{\mu}_i^\prime)\|_1  & = \| \bigoplus_{C\in S} \Phi(\bar{\mu}_i,C)-\bigoplus_{C\in S}
\Phi(\bar{\mu}_i^\prime,C^\prime)\|_1\\
& = \sum_{C\in S}\|\Phi(\bar{\mu}_i,C)-\Phi(\bar{\mu}_i^\prime,C^\prime)\|_1\\
& \leq L\Delta,
\end{aligned}
\end{equation*}
where $L=mt(M^{\max}+kC_0^{\max})$.

\end{proof}
\subsection{Proof of Theorem \ref{thm: sep}.}
\textbf{Theorem \ref{thm: sep}.}     For any given codebook $C$ and two EPDs ($\bar{\mu},\bar{\nu}$) with the same total mass, it holds that
    \begin{equation*}
            \|\Phi(\bar{\mu},C)-\Phi(\bar{\nu},C)\|_1  \geq  \frac{W_1(\bar{\mu},\bar{\nu})-W_1(\bar{\mu},\hat{\mu}(C))- W_1(\bar{\nu},\hat{\nu}(C))}{d_{max}(C)},
    \end{equation*}
    where $d_{max}(C)=\max_{c_i,c_j\in C}\|c_i-c_j\|_2$ is the diameter of $C$  and $W_1$ refers to  Wasserstein distance.
\begin{proof}
    By the triangle inequality of Wasserstein distance, $\forall$ codebook $C$, it holds that 
\begin{equation*}
    \begin{aligned}
        W_1(\bar{\mu},\bar{\nu}) & \leq W_1(\bar{\mu},\hat{\mu}(C)) +W_1(\hat{\mu}(C),\bar{\nu})\\
        & \leq W_1(\bar{\mu},\hat{\mu}(C))+W_1(\hat{\mu}(C),\hat{\nu}(C))+W_1(\bar{\nu},\hat{\nu}(C)).
    \end{aligned}
\end{equation*}
By rearranging the terms, we have 
\begin{equation*}
    W_1(\hat{\mu}(C),\hat{\nu}(C))\geq W_1(\bar{\mu},\bar{\nu})- W_1(\bar{\mu},\hat{\mu}(c))-W_1(\bar{\nu},\hat{\nu}(c)). 
\end{equation*}
By the definition of optimal partial transport metric $W_1(\hat{\mu}(c),\hat{\nu}(c))=\inf_{\pi}\iint_{\Omega\times\Omega} \|x-y\|_2d\pi$ corresponds to the minimal cost of all the transportation plans. So it is the lower bound of the cost of a particular transportation plan. We now give a designed transportation plan that link the cost and the distance between Vreps $\|\Phi(\bar{\mu},C)-\Phi(\bar{\nu},C)\|_1$.

A particular transportation plan from $\hat{\mu}(C)$ to $\hat{\nu}(C)$: since $\hat{\mu}(C)$ and $\hat{\nu}(C)$ have the same support $C$, we keep the mass of each support point $c_i$ as $m(c_i) =min(\hat{\mu}(C)(c_i),\hat{\nu}(C)(c_i))$ fixed in the transportation. We move the extra mass in $\{c_i:m(c_i)<\hat{\mu}(C)(c_i)\}$ to the mass-vacant area $\{c_i:m(c_i)>\hat{\mu}(C)(c_i)\}$. 

Since the sum of mass that is needed to be moved is less than $\|\Phi(\bar{\mu},C))-\Phi(\bar{\nu},C)\|_1$ and the cost of each transportation is less than or equal to $d_{max}(C)$, the total cost of this transportation plan is less than or equal to $\|\Phi(\bar{\mu},C)-\Phi(\bar{\nu},C)\|_1 \cdot d_{max}(C)$. 

Based on that, the cost of this particular transportation plan is larger than $W_1(\hat{\mu}(C),\hat{\nu}(C))$. So it holds that 
\begin{equation*}
\begin{aligned}
    \|\Phi(\bar{\mu},&C)-\Phi(\bar{\nu},C)\|_1 \cdot d_{max}(C)\geq  W_1(\hat{\mu}(C),\hat{\nu}(C)) \\
     \geq & W_1(\bar{\mu},\bar{\nu})- W_1(\bar{\mu},\hat{\mu}(C))-W_1(\bar{\nu},\hat{\nu}(C)).
\end{aligned}
\end{equation*}

\end{proof}

\subsection{Proof of Theorem \ref{thm: sep2}}
\textbf{Theorem \ref{thm: sep2}.}     For any given codebook $C$ and two EPDs ($\bar{\mu},\bar{\nu}$) with the same total mass, it holds that
    \begin{equation*}
            \|\Phi(\bar{\mu},C)-\Phi(\bar{\nu},C)\|_1  \leq  \frac{k}{d_{min}(C)}[W_1(\bar{\mu},\bar{\nu})+W_1(\bar{\mu},\hat{\mu}(C))+ W_1(\bar{\nu},\hat{\nu}(C))],
    \end{equation*}
    where $d_{min}(C)=\min_{c_i,c_i\in C}\|c_i-c_j\|_1$.

\begin{proof}
    By the triangle inequality of Wasserstein distance, $\forall$ codebook $c$, it holds that
    \begin{equation*}
    \begin{aligned}
        W_1(\hat{\mu}(C),\hat{\nu}(C)) & \leq W_1(\hat{\mu}(C),\bar{\mu}) +W_1(\bar{\mu},\hat{\nu}(C))\\
        & \leq W_1(\hat{\mu}(c),\bar{\mu})+W_1(\bar{\mu},\bar{\nu})+W_1(\hat{\nu}(C),\bar{\nu}).
    \end{aligned}
\end{equation*}

For $W_1(\hat{\mu}(C),\hat{\nu}(C))$, the transportation cost at least contains $\max_{i}|\hat{\mu}(C)(c_i)-\hat{\nu}(C)(c_i)|\cdot d_{min}(C)$. Since $\|\Phi(\bar{\nu},C)-\Phi(\bar{\nu},C)\|_{\infty}=\max_{i}|\hat{\mu}(C)(c_i)-\hat{\nu}(C)(c_i)|$, it holds that 
\begin{equation*}
\begin{aligned}
    W_1(\hat{\mu}(C),\hat{\nu}(C))\geq & \|\Phi(\bar{\nu},C)-\Phi(\bar{\nu},C)\|_{\infty}\cdot d_{min}(C)\\
     \geq &\frac{\|\Phi(\bar{\nu},C)-\Phi(\bar{\nu},C)\|_1}{k}\cdot d_{min}(C).
\end{aligned}
\end{equation*}

Combined with the first inequality derived from triangle inequality of Wasserstein distance, we obtain the final result:
    \begin{equation*}
    \|\Phi(\bar{\mu},C)-\Phi(\bar{\nu},C)\|_1  \leq  \frac{k}{d_{min}(C)}[W_1(\bar{\mu},\bar{\nu})+W_1(\bar{\mu},\hat{\mu}(C))+ W_1(\bar{\nu},\hat{\nu}(C))].
    \end{equation*}
\end{proof}

\section{Additional Notes}

\subsection{Effect of the assumption due to the use of point transformation function $f$.}
\label{appendix: example distribution}

Here we given two special EPDs (distribution of topological features) $\mu$ and $\nu$. In this constructed example, the Voronoi histogram with a suitable codebook better preserves the relevant mass displacement than the selected PI/PS parameterizations. This illustrates one possible trade-off of partition-based aggregation, not a general superiority result.

For an EPD $\mu$, a uniform distribution supported on a rectangle with the length of one side of the rectangle being 1 and the total mass being 1, we move $\mu$ along the diagonal $\partial\Omega$ by $1-\alpha$ to get $\nu$, as shown in Figure \ref{fig: example EPD} (a). We denote the support of $\mu$ ($\nu$) as $\mathcal{S}_\mu$ ($\mathcal{S}_\nu$) and the intersection area as $A=\mathcal{S}_\mu \cap \mathcal{S}_\nu$.
Assume that the translation from $\mu$ to $\nu$ is very small ($\alpha$ is close to 1) so that the Wasserstein distance between $\mu$ and $\nu$ is $1-\alpha$, i.e $W(\mu,\nu)=1-\alpha$.

For our proposed codebook-based representation $\Phi$, once the size $k$ of codebook $C$ is very large ($\partial\Omega\notin C$), given the histogram nature of $\Phi$ with each Voronoi Cell as bin, it holds that $\|\Phi(\mu,C)-\Phi(\nu,C)\|_1/\|\Phi(\mu,C)\|_1=2(1-\alpha)=2W(\mu,\nu)$ for $C$ sampled from $\mu$ or $\nu$. So for Vrep, it holds that $\|\hat{\Phi}(\mu)-\hat{\Phi}(\nu)\|_1/\|\hat{\Phi}(\mu)\|_1=2(1-\alpha)=2W(\mu,\nu)$.

For linear representation $\Psi$ like PI and PS, it holds that

\begin{equation}
    \begin{aligned}
        & \frac{\|\Psi(\mu)-\Psi(\nu)\|_1}{\|\Psi(\mu)\|_1}\\
        = & \frac{\|\sum_{x\in \mathcal{S}_{\mu}\setminus A}f(x)-\sum_{y\in \mathcal{S}_{\nu}\setminus A}f(y)\|_1}{\|\Psi(\mu)\|_1}\\
        = &  \frac{\|\sum_{x\in\mathcal{S}_\mu\setminus A }[f(x)-f(\eta(x))]\|_1}{\|\Psi(\mu)\|_1}\\
        \leq & \frac{\sum_{x\in\mathcal{S}_\mu\setminus A }\|f(x)-f(\eta(x))\|_1}{\|\Psi(\mu)\|_1}\\
        \leq & \frac{\sum_{x\in\mathcal{S}_\mu\setminus A }\|f(x)\|_1+\|f(\eta(x))\|_1}{\|\Psi(\mu)\|_1},
        \end{aligned}
    \end{equation}

    where $\eta$ is a bijection between $\mathcal{S}_\mu \setminus A$ and $\mathcal{S}_\nu \setminus A$, for $x\in\mathcal{S}_\mu \setminus A, \eta(x)$ is the mirror point that is symmetrical w.r.t. the middle black dashed line in Figure \ref{fig: example EPD} (b), and $f$ is the point transformation function. Assume that for any $x,y\in \mathcal{S}_\mu\cup \mathcal{S}_\nu$, it holds that $\|f(x)\|_1=\|f(y)\|_1$, where $f(x)$ and $f(y)$ are the transformed functions.

    Given the open half plane $\Omega=\{(t_1,t_2)\in\mathbb{R}^2|t_2> t_1\}$ and $x=(x_b,x_d)\in \Omega$, for PI, $f(x)(\cdot) = exp(-\|\cdot-x\|_2^2/2\sigma^2)$. For PS, $f(x)(\cdot)=max\{0,x_d-|\cdot-x_b|\}$, $x=(x_b,x_d)$. For both PI and PS, there might exist an area $B$ such that for $z\in B$, $\psi(x)(z)$ and $\psi(\eta(x))(z)$ are both larger than 0, as shown in Figure \ref{fig: example EPD} (c) and (d). Hence, the last line of inequality (1) can not take the equal sign. In consequence, it holds that \[\frac{\|\Psi(\mu)-\Psi(\nu)\|_1}{\|\Psi(\mu)\|_1}<\frac{\sum_{x\in\mathcal{S}_\mu\setminus A }\|f(x)\|_1+\|f(\eta(x))\|_1}{\|\Psi(\mu)\|_1}=2(1-\alpha)=2W(\mu,\nu).\]

    Given the specific distributions $\mu$ and $\nu$, Vrep can truthfully reflect Wasserstein distance while PS and PI can not because they assume the distribution is the summation of Gaussian distributions or piecewise linear functions.

\begin{figure}[t]
    \centering
\includegraphics[scale=0.25]{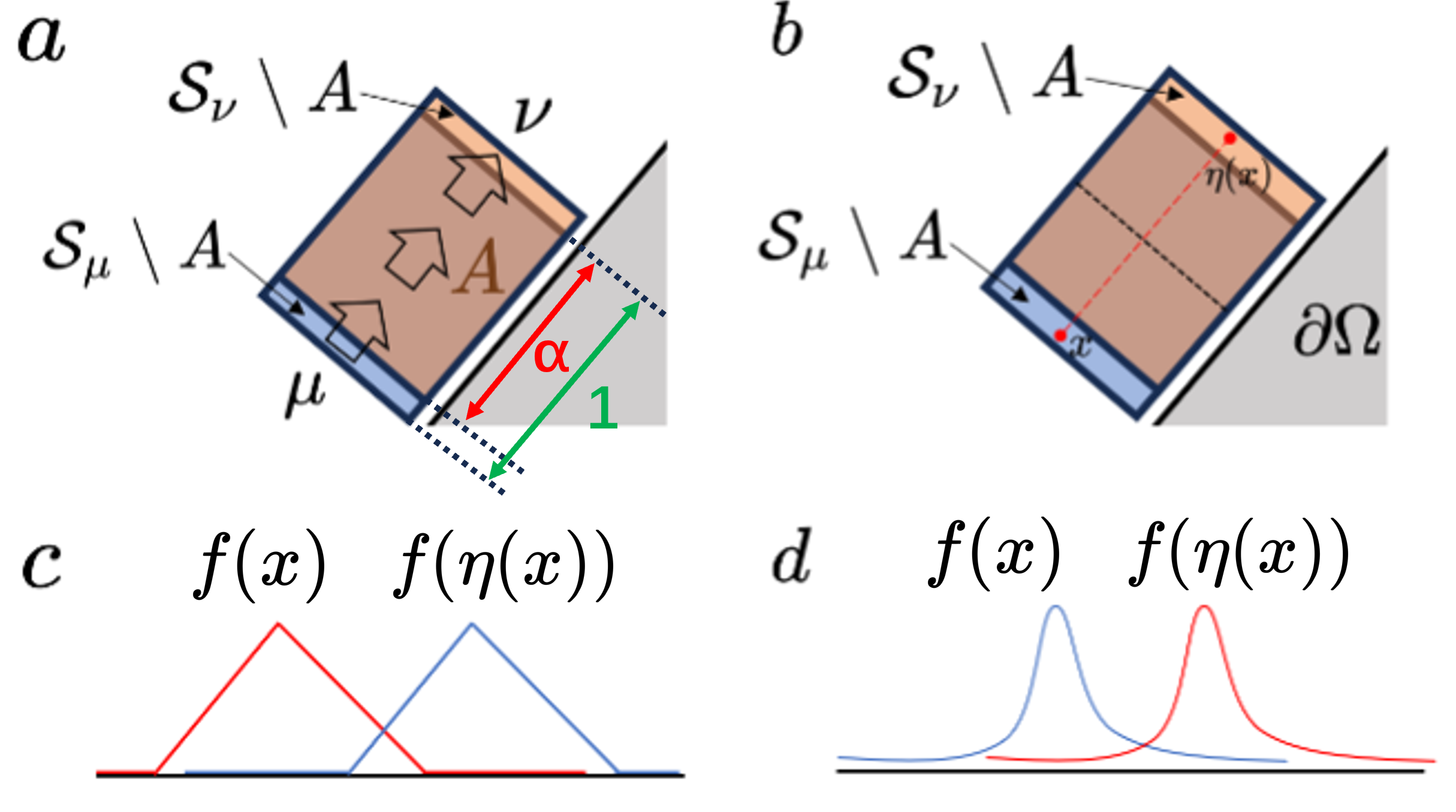}
    \caption{(a) Distribution $\nu$ is obtained via the translation of $\mu$. The arrows indicate the direction of translation. The blue and red arrows are for marking the length of rectangle. (b) Illustration of the bijection map $\eta$. (c) Illustration of the point transformation function of PS. (d) Illustration of the point transformation function of PI.}
\label{fig: example EPD}
\end{figure}

\subsection{Hyperparameter $t$ has no effect on change in Vrep (Vrep$_d$) representation.}
\label{appendix: t effect}
According to Theorem \ref{theorem: stab}, we have that for an EPD dataset $\{\bar{\mu}_1,...,\bar{\mu}_m\}$, and$\Delta$-perturbation of every EPD $\bar{\mu}_i$, it holds that
    \begin{equation*}
        \|\hat{\Phi}(\bar{\mu}_i)-\hat{\Phi}(\bar{\mu}_i^\prime)\|_1\leq L\Delta,
    \end{equation*}
    where $L=mt(M^{\max}+kC_0^{\max})$ is a constant determined by the hyperparameters of Vrep and $\bar{\mu}_i^\prime$ is the perturbed version of $\bar{\mu}_i$.

    Combined with the fact that for any EPD $\bar{\mu}$, $\|\hat{\Phi}(\bar{\mu})\|_1=mt$, we have that 
    \begin{equation*}
    \delta_{\bar{\mu}_i}=\frac{\|\hat{\Phi}(\bar{\mu}_i)-\hat{\Phi}(\bar{\mu}_i^\prime)\|_1}{\|\hat{\Phi}(\bar{\mu}_i)\|_1}\leq \frac{L\Delta}{mt}=(M^{\max}+kC_0^{\max})\Delta.
    \end{equation*}
    So hyperparameter $t$ has no effect on change in Vrep (Vrep$_d$) representation.

\subsection{Data-dependence of Vrep}
\label{sect: data-dependence}
We describe the data-dependent nature of Vrep (and Vrep$_d$). By data-dependence, we mean that the representation of an EPD $\bar{\mu}$ depends more than itself: the codebook $C$ is sampled from another EPD $\bar{\nu}$ and then used to build representations of $\bar{\nu}$ and $\bar{\mu}$. Current vectorization methods just average the vectors of sampled PDs in each EPD, and they do not possess this data-dependence property. Given an EPD dataset with just two EPDs $\bar{\mu}$ and $\bar{\nu}$, our objective here is to clarify precisely what factors $\|\Phi(\bar{\mu},C)-\Phi(\bar{\nu},C))\|_1$ depends on when $C$ is sampled from one particular EPD $\bar{\nu}$.

\begin{theorem}
    \label{thm: sep2}
    For any given codebook $C$ with size $k$ and two EPDs ($\bar{\mu},\bar{\nu}$), it holds that
    \begin{equation*}
            \|\Phi(\bar{\mu},C)-\Phi(\bar{\nu},C)\|_1 \leq  B_U(\bar{\mu},\bar{\nu},C),
    \end{equation*}
    where $d_{min}(C)=\min_{c_i,c_j\in C,i\neq j}\|c_i-c_j\|_2$ and the upper bound $B_U(\bar{\mu},\bar{\nu},C)=\frac{k}{d_{min}(C)}[W_1(\bar{\mu},\bar{\nu})+W_1(\bar{\mu},\hat{\mu}(C))+W_1(\bar{\nu},\hat{\nu}(C))]$.
\end{theorem}

\begin{figure}[h]
    \centering
\includegraphics[scale=0.45]{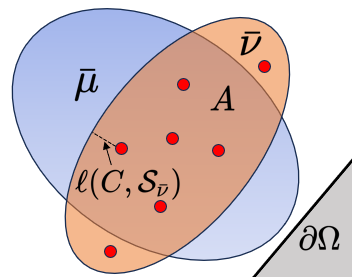}
    \caption{Two measures $\bar{\mu},\bar{\nu}$ on $\Omega$ and codebook $C$ (red points) sampled from $\bar\nu$.}
\label{fig:data_dependent}
\end{figure}

We provide an upper bound $B_U$ for the $l_1$ distance between $\Phi(\bar{\mu},C)$ and $\Phi(\bar{\nu},C)$. A larger $B_U$ would indicate a possible larger $l_1$ distance. Next we focus on the case where $C$ is sampled from $\bar{\nu}$ and its effect on $B_U$.

If the codebook is sampled from $\bar{\nu}$, the key of the analysis is to disentangle $\bar{\mu}$ from $C$ in the $W_1(\bar{\mu},\hat{\mu}(C))$ term of $B_U$. Assume the supports of $\bar{\mu}$ and $\bar{\nu}$ are both compact, denoted as $\mathcal{S}_{\bar{\mu}},\mathcal{S}_{\bar{\nu}}$ and let $A=\mathcal{S}_{\bar{\mu}}\cap \mathcal{S}_{\bar{\nu}}$. If $C$ is sampled from $\bar{\nu}$, then for $W_1(\bar{\mu},\hat{\mu}(C))$, at least mass of $\bar{\mu}(\mathcal{S}_{\bar{\mu}}\setminus A)=\bar{\mu}(\Omega)-\bar{\mu}(A)$ need to be transported to $C\subset A$, as shown in Figure \ref{fig:data_dependent}. So we can obtain a simple lower bound of $B_U$ as
\begin{equation*}
        B_U(\bar{\mu},\bar{\nu},C)\geq  \frac{k}{d_{min}(C)}[W_1(\bar{\mu},\bar{\nu})+W_1(\bar{\nu},\hat{\nu}(C))+\ell(C,\partial(\mathcal{S}_{\bar{\nu}}))(\bar{\mu}(\Omega)-\bar{\mu}(A))],
\end{equation*}
where $\ell(X,Y)=\min_{x\in X,y\in Y}\|x-y\|_2$. \textbf{This lower bound indicates that when codebook $C$ is sampled from $\bar{\nu}$, two factors \footnote{Note that these two factors also affect the lower bound in Theorem 0.5.} (i) less mass of $\bar{\mu}$ in the intersection area $A$ and (ii) a larger $\ell(C,\partial(\mathcal{S}_{\bar{\nu}}))$  indicate a higher lower bound of $B_U$. Consequently, the distance between the codebook-based representation of $\bar{\mu}$ and $\bar{\nu}$ w.r.t. codebook $C$ are more likely to be large. }
    
Combined the results from Theorem 0.5 and \ref{thm: sep2}, if codebook $C\subset \mathcal{S}_\nu$, we have 
{
\small
\begin{equation}
\label{equation: upper and lower bound}
\begin{aligned}
    \frac{1}{d_{max}(C)} [W_1(\bar{\mu},\bar{\nu})-W_1(\bar{\mu},\hat{\mu}(C))- W_1(\bar{\nu},\hat{\nu}(C))]\leq 
    \|\Phi(\bar{\mu},C)-\Phi(\bar{\nu},C)\|_1 \\
    \leq \frac{k}{d_{min}(C)}[W_1(\bar{\mu},\bar{\nu})+W_1(\bar{\mu},\hat{\mu}(C))+ W_1(\bar{\nu},\hat{\nu}(C))].
\end{aligned}
\end{equation}
}
We denote the lower bound of $\|\Phi(\bar{\mu},C)-\Phi(\bar{\nu},C)\|_1$ as $B_L$ and the upper bound as $B_U$. With the result $W_1(\bar{\mu},\hat{\mu}(C))\geq\ell(C,\partial(\mathcal{S}_{\bar{\nu}}))(\bar{\mu}(\Omega)-\bar{\mu}(A))$, we can have a upper bound for $B_L$ and lower bound for $B_U$:
\begin{equation*}
    B_L\leq\frac{1}{d_{max}(C)} [W_1(\bar{\mu},\bar{\nu})-\ell(C,\partial(\mathcal{S}_{\bar{\nu}}))(\bar{\mu}(\Omega)-\bar{\mu}(A))- W_1(\bar{\nu},\hat{\nu}(C))],
\end{equation*}
\begin{equation*}
    B_U\geq\frac{k}{d_{min}(C)}[W_1(\bar{\mu},\bar{\nu})+\ell(C,\partial(\mathcal{S}_{\bar{\nu}}))(\bar{\mu}(\Omega)-\bar{\mu}(A))+ W_1(\bar{\nu},\hat{\nu}(C))].
\end{equation*}
The effect of $\ell(C,\partial(\mathcal{S}_{\bar{\nu}}))(\bar{\mu}(\Omega)-\bar{\mu}(A))$ has coefficient $\frac{k}{d_{min}(C)}$ for $B_U$ and $\frac{1}{d_{max}(C)}$ for $B_L$. So the major effect of $\ell(C,\partial(\mathcal{S}_{\bar{\nu}}))(\bar{\mu}(\Omega)-\bar{\mu}(A))$ lies in $B_U$.




\subsection{Advantage of normalization}
\label{appendix: norm}
As pointed out by \cite{wu2024estimation}, normalization is more appropriate when the number of points in PD is not of direct interest but their spatial distribution is. This is typically the case when the PDs contain many points or are obtained from large random filtrations, so that the number of points in PD will mostly account for noisy topological fluctuations due to sampling. 
\subsection{Limitation}
\label{appendix: limitation}
Vrep is a finite-dimensional histogram representation of Expected Persistence Diagrams (EPDs), and therefore it is inherently lossy. Its main source of information loss comes from aggregating all mass inside the same Voronoi cell into a single coefficient. As a result, Vrep can preserve coarse mass allocation across cells, but it discards within-cell geometry. When the task-relevant variation is dominated by fine local shifts or changes in the shape of the distribution inside a cell, smooth vectorizations such as PI, PS, or kernel methods may be more sensitive.

The quality of Vrep also depends on the choice of codebooks. Codebook size, sampling strategy, and Voronoi cell geometry determine the resolution of the representation. A small or poorly placed codebook may merge distinct topological regions, while a very large codebook increases representation dimension and can reduce stability under perturbations. Our experiments compare default, persistence-weighted, and uniform codebook sampling, but optimal or task-adaptive codebook construction remains open.


Vrep also relies on normalization of EPD mass in our theoretical analysis and main experimental setting. This normalization makes Wasserstein distance directly applicable by giving EPDs equal total mass, but it removes information about the total number or total mass of topological features. In tasks where feature count, total persistence mass, or density of topological events is predictive, normalization may discard useful signal. In such cases, unnormalized variants or additional mass-based summary features may be preferable.


Finally, our empirical evaluation focuses on datasets where topological information is expected to be useful. The results show that Vrep is competitive in these settings and can be computationally attractive when many sampled persistence diagrams are used to form an EPD. However, the experiments do not establish uniform superiority over PI, PS, PL, PWGK, SWK, or supervised topological representations. Broader evaluation on additional domains, larger datasets, and downstream tasks is needed to better characterize when adaptive Voronoi aggregation is preferable to smooth functional or kernel-based representations.

\section{Experiment Setting \& Summary of Datasets}
\label{append: exp setting}
\subsection{Detailed settings}

For the Persistence Image (PI), the resolution is set to $10\times 10$, resulting in a 100-dimensional vector. The bandwidth is in the range of $[0.0001,0.001,0.01,0.1]$. For Persistence Landscape (PL), the kmax hyperparameter $k$ is in the range of $[2,4,6,8]$. The resolution of Persistence Landscape and Persistence Silhouette (PS) is 100 for alignment with Persistence Image. The weight function of PI and PS is the square of persistence. For Persistence Weighted Gaussian Kernel, the bandwidth is also in $[0.0001,0.001,0.01,0.1]$. For Vrep (Vrep$_d$), we set $t=10$ and $k$ is in the range of $[2,4,6,8,10,12,14,16,18,20]$. We set the subsample size of each EPD support to be 50, i.e. $|\mathcal{S}_{\bar\mu}|=50$, The dimension of topological features is also viewed as a hyperparameter. We only consider 0-dimensional or 1-dimensional PD, i.e. connected components or rings.


For the computation of EPD, in CAD, each sampled PD is computed from $2\%$ randomly sampled points from point cloud. For Protein, each sampled PD is computed from 50 randomly sampled points. For time series data, each sampled PD is computed from $50\%$ randomly sampled points from point cloud. The number of sampled PD in each EPD is 50. For Random Forest Classifier, we set the size of the forest to be 100 and use Gini impurity to measure the quality of a split.


In the study on the codebook choice for Vrep and Vrep$_d$, for a dataset of EPDs, for the uniform choice, we choose to sample codebook from a uniform distribution supported on a rectangle area that covers all the EPDs in the dataset. 

For scaleup test on the number of sampled
PDs in each EPD, we use the CAD dataset with no noise. The time cost of each vectorization is computed on the entire dataset. We report the mean time cost of over different hyperparameter choices.

For PointNet, we set batch size to be 8. We use Adam optimizer and set the learning rate to be 0.001. We train the model with 50 epochs in all the datasets except for the Protein dataset, which uses 400 epochs to ensure convergence. In addition, PointNet requires the input to be in the same size but different point clouds in Protein and CAD datasets have different number of points. So for protein, we fix the input size to be 30. For CAD, the input size is 2500. 

All methods used in the following experiments were run on a machine
with i7-12700K CPU (3.60 GHz), 128 GB RAM and NVIDIA GeForce RTX 3060 GPU. PD is computed purely on CPU. PointNet runs on GPU. All the vectorization methods are conducted on CPU.

\subsection{Summary of datasets}
The summary of datasets is provided in Table \ref{tab: dataset}. An example point cloud in CAD dataset is provided in Figure \ref{fig: example point cloud}. For each protein, a 
cross-correlation matrix $\mathbb{C}$ is provided. Then we use distance matrix $\mathbb{D}$, where $\mathbb{D}_{i,j}=1-|\mathbb{C}_{i,j}|$, as input to the EPD computation. CAD is a large-scale
dataset of 3D objects of mechanical components. We choose a preselected version \cite{cao2022approximating} of this dataset which contains two classes `Bearing' and `Motor'. For time series data (Beef, BirdChicken, DistalPhalanxTW, Earthquakes, ECG200), we follow the setting in \cite{dong2024persistence} and transform a time series into point cloud via Time-Delaying embedding \cite{seversky2016time}. 

\begin{table}[h]
\centering
\caption{Summary of the point cloud datasets. In each each dataset: the size (number of point clouds), the minimum and maximum  numbers of points in one point cloud and the number of classes.}
\label{tab: dataset}
\vspace{8pt}
\begin{tabular}{ | c | c | c | c | c | } 
  \hline
  dataset & size & \# min & \# max & \# classes\\\hline
  Protein & 99 & 32 & 691 & 5 \\
  CAD & 126 & 30469 & 206858 &2\\
  Beef & 60 & 468 & 468 & 5 \\
  BirdChicken & 40 & 510 & 510 &2\\
  DPTW & 539 & 78 &78 & 6 \\
  Earthquakes & 461 & 510 &510 & 2 \\
  ECG200 & 200 & 94 & 94 &2 \\\hline
\end{tabular}
\end{table}
\begin{figure}[h]
    \centering
\includegraphics[scale=0.4]{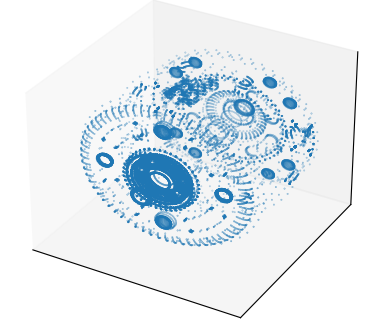}
    \caption{A point cloud in CAD dataset.}
\label{fig: example point cloud}
\end{figure}


\section{Additional Experiments}
\subsection{Hyperparameter sensitivity of Vrep (Vrep$_d$) with different codebook choices}
\label{appendix: sens}

\begin{figure}[h!]
\centering
\includegraphics[width=0.4\textwidth]{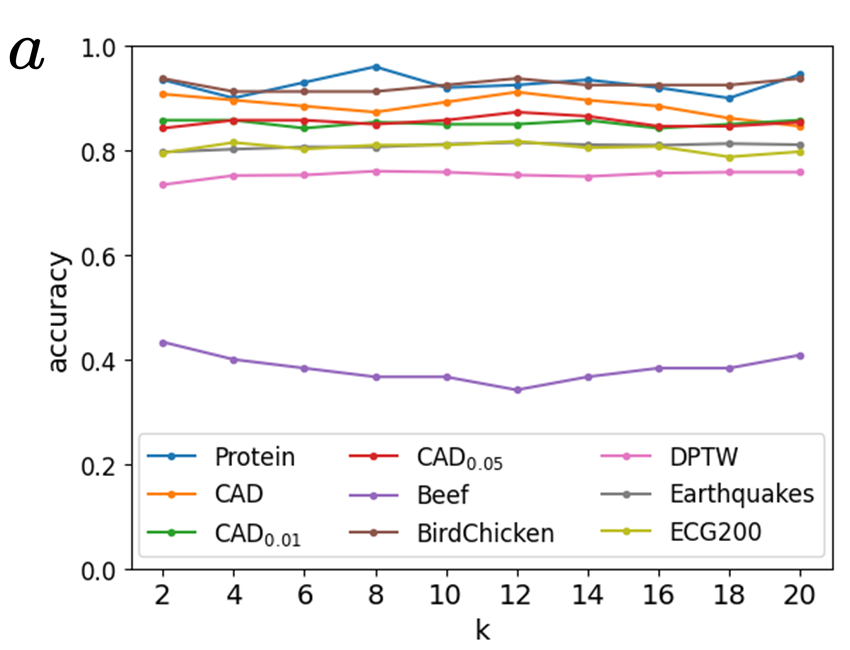}
\includegraphics[width=0.4\textwidth]{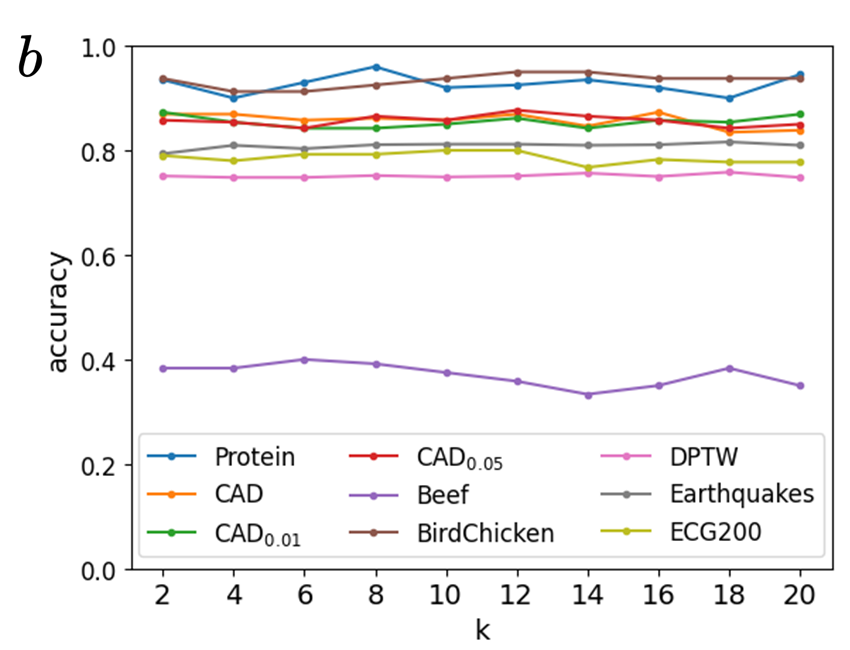}
\caption{Hyperparameter sensitivity of Vrep (a) and Vrep$_d$ (b) with the default codebook choice w.r.t. $k$. By default, the value of $t$ is 10. 
}
\label{fig: sens}
\end{figure}

\begin{figure}[h!]
\centering
\includegraphics[width=0.4\textwidth]{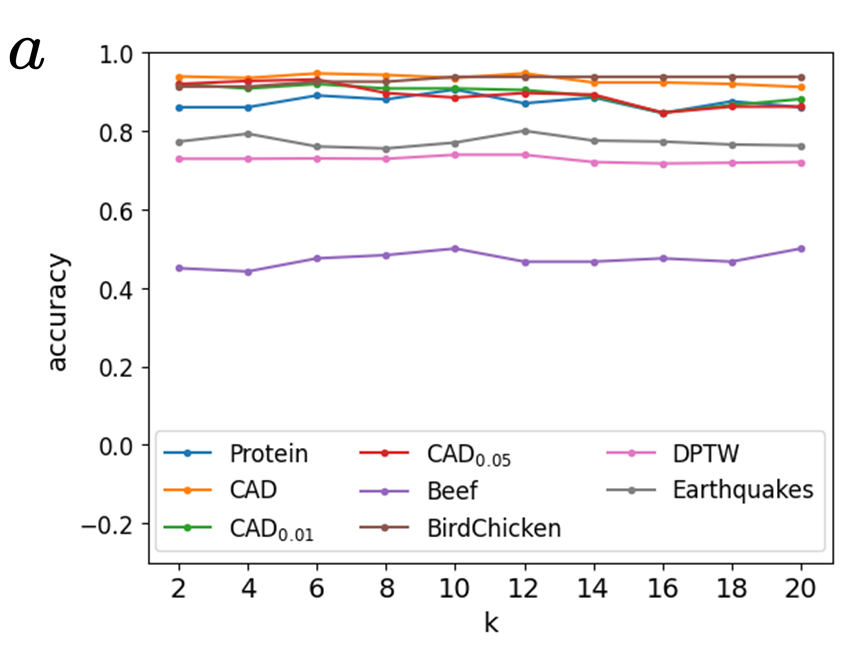}
\includegraphics[width=0.4\textwidth]{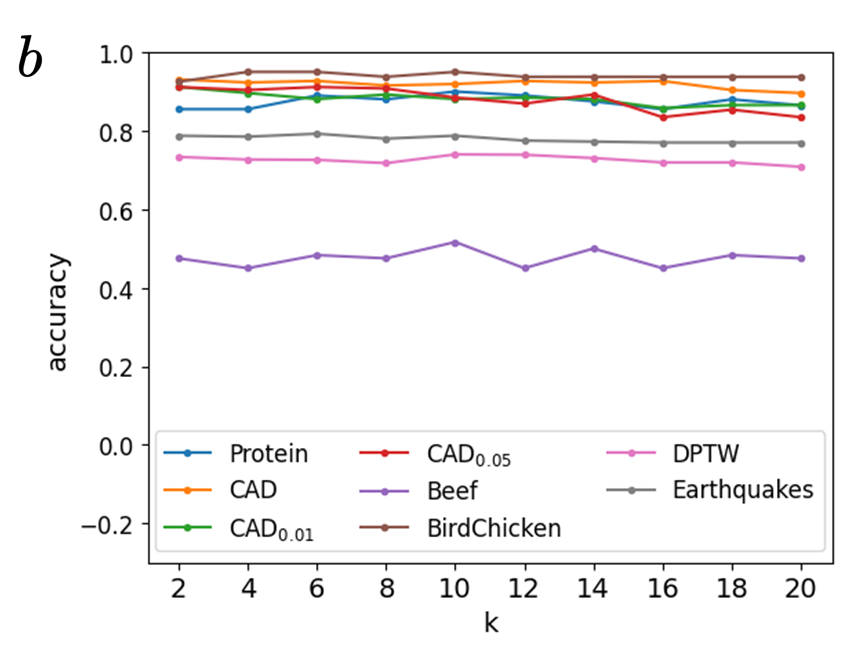}
\caption{Hyperparameter sensitivity of Vrep (a) and Vrep$_d$ (b) with the persistence-weighted codebook choice w.r.t. $k$. By default, the value of $t$ is 10. 
}
\label{fig: sens pw}
\end{figure}

\begin{figure}[h!]
\center
\includegraphics[width=0.4\textwidth]{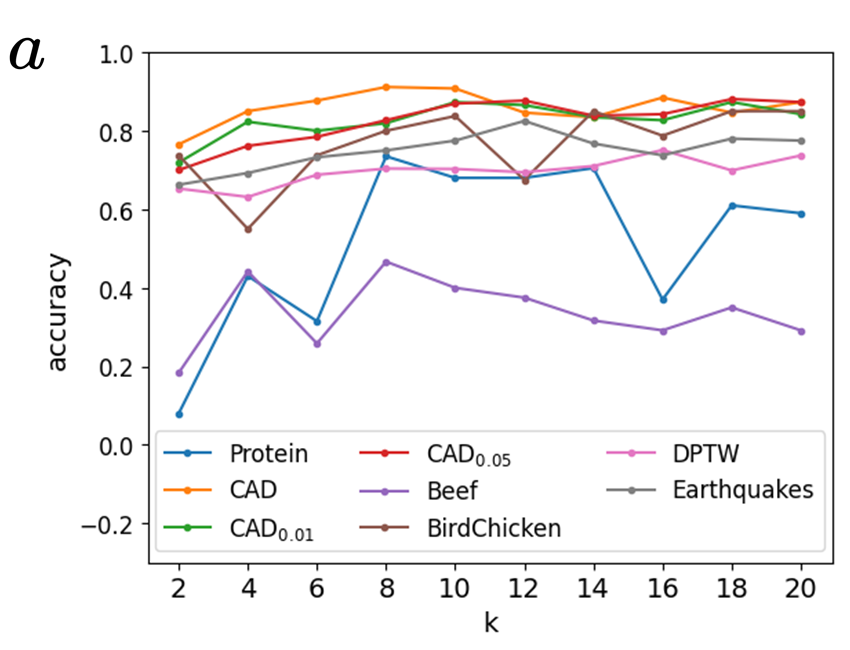}
\includegraphics[width=0.4\textwidth]{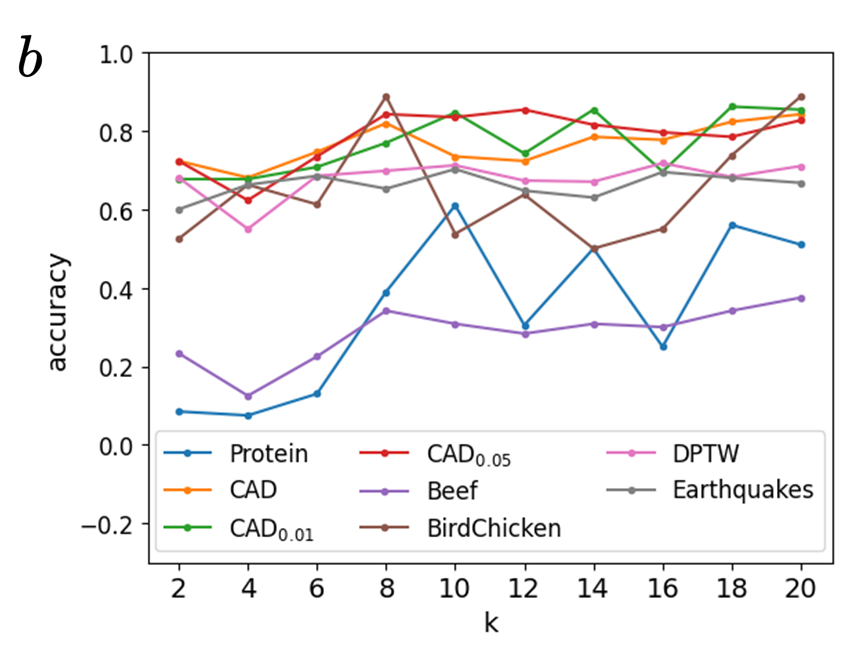}
\caption{Hyperparameter sensitivity of Vrep (a) and Vrep$_d$ (b) with the uniform codebook choice w.r.t. $k$. By default, the value of $t$ is 10. 
}
\label{fig: sens uniform}
\end{figure}

The hyperparameter sensitivity results of Vrep (Vrep$_d$) with persistence-weighted and uniform codebook choice are shown in Figure \ref{fig: sens pw} and \ref{fig: sens uniform}, respectively. The uniform case is less stable w.r.t. $k$ because the codebook area of the uniform case, i.e. the minimal box above the diagonal that bounds all the EPDs, is larger than the codebook area of the persistence-weighted case, which is restricted to the support of EPD. And when a new point is add to the codebook, a larger codebook area contains more variability for the location of this new point. This variability leads to the instability.

\subsection{Comparison with Persistence Weighted Gaussian Kernel and Sliced Wasserstein Kernel}
\label{appendix: PWGK}

Persistence Weighted Gaussian Kernel (PWGK) is a kernel for PD. When employed on EPD, it measures the similarity of two EPDs $\bar{\mu}_1=\frac{1}{n}\sum_{i=1}^{n}\mu^1_i$ and $\bar{\mu}_2=\frac{1}{n}\sum_{i=1}^{n}\mu^2_i$ as $\frac{1}{n}\sum_{i=1}^n\text{PWGK}(\mu^1_i,\mu^2_i)$. When compared with PWGK, Vrep (Vrep$_d$) is used in a linear kernel fashion, i.e. the similarity of two EPDs is defined to be $\left \langle \hat{\Phi}(\bar{\mu}_1),\hat{\Phi}(\bar{\mu}_2)\right \rangle $. The average accuracy and time cost over 10 random split are shown in Tables \ref{tab: exp PWGK} and \ref{tab: exp PWGK2}. The time cost of computing PWGK for EPD is too large for us to conduct scaleup test, so here we only report its average time cost for the datasets. The same goes for Sliced Wasserstein Kernel (SWK).

\begin{table}[h]
    \centering
    \caption{Average accuracy and time used to computing kernel value of Vrep, Vrep$_d$, PWGK and SWK. Here we only consider 0-dimensional PD and EPD.}
    \label{tab: exp PWGK}
    \vspace{8pt}
    \resizebox{0.95\textwidth}{!}{
    \begin{tabular}{|c|c|c|c|c|c|c|c|c|c|c|}
    \hline
        ~ & \multicolumn{2}{c|}{Vrep} & \multicolumn{2}{c}{Vrep$_d$} & \multicolumn{3}{|c|}{PWGK} & \multicolumn{3}{|c|}{SWK} \\ \cline{2-11}
        ~ & EPD & time (s) & EPD & time (s) & PD & EPD & time (s) & PD & EPD & time (s) \\\hline
        Protein & 0.975 & 4.24 & 0.975 & 2.42 & 0.880 & 0.895 & 151.35  & 0.955 & 0.865 & 25.97 \\
        CAD & 0.846 & 6.76 & 0.823 & 7.64 & $\backslash$ & 0.912 & 1339.71& $\backslash$& 0.785 & 3606.91 \\
        CAD$_{0.01}$ & 0.835 & 4.16 & 0.827 & 4.86 & $\backslash$ & 0.900 & 1650.24 & $\backslash$ & 0.785 & 3434.44 \\
        CAD$_{0.05}$ & 0.838 & 4.16 & 0.842 & 4.87 & $\backslash$& 0.877 & 2120.78 & $\backslash$& 0.662 & 5200.52 \\
        Beef & 0.483 & 0.981 & 0.5 & 1.14 & 0.250 & 0.408 & 82.23 & 0.267 & 0.150 & 165.99 \\
        BirdChicken & 0.938 & 0.44 & 0.938 & 0.51 & 0.938 & 0.938 & 43.02 & 0.888 & 0.900 & 84.22 \\
        DPTW & 0.646 & 77.98 & 0.642 & 90.56 & 0.693 & 0.737 & 2056.88 & 0.744 & 0.750 & 1733.85 \\
        Earthquakes & 0.809 & 56.07 & 0.786 & 65.09 & 0.803 & 0.787 & 4322.23 & 0.803 & 0.803 & 7716.87\\
        ECG200 & 0.758 & 10.24 & 0.753 & 11.87 & 0.680 & 0.740 & 261.57 & 0.807 & 0.798 & 302.59\\\hline
    \end{tabular}}
\end{table}

\begin{table}[h]
    \centering
    \caption{Average accuracy and time used to computing kernel value of Vrep, Vrep$_d$ PWGK and SWK. Here we only consider 1-dimensional PD and EPD.}
    \label{tab: exp PWGK2}
    \vspace{8pt}
    \resizebox{0.95\textwidth}{!}{
    \begin{tabular}{|c|c|c|c|c|c|c|c|c|c|c|}
    \hline
        ~ & \multicolumn{2}{c|}{Vrep} & \multicolumn{2}{c|}{Vrep$_d$} & \multicolumn{3}{c|}{PWGK} & \multicolumn{3}{c|}{SWK}\\ \cline{2-11}
        ~ & EPD & time (s) & EPD & time (s) & PD & EPD & time (s) & PD & EPD & time (s) \\\hline
       Protein & 0.915  & 2.18  & 0.895  & 2.43  & 0.910  & 0.330  & 86.12 & 0.520 & 0.165 & 45.37   \\ 
        CAD & 0.877  & 4.34  & 0.865  & 5.03  & $\backslash$ & 0.919  & 162.33 & $\backslash$ & 0.888 & 347.89  \\ 
        CAD$_{0.01}$ & 0.846  & 4.33  & 0.838  & 5.06  & $\backslash$ & 0.919  & 163.97 & $\backslash$ & 0.900 & 386.69  \\ 
        CAD$_{0.05}$ & 0.858  & 4.24  & 0.846  & 4.98  & $\backslash$ & 0.896  & 179.90 & $\backslash$ & 0.888 & 447.42  \\ 
        Beef & 0.550  & 0.93  & 0.517  & 1.10  & 0.367  & 0.300  & 17.29 & 0.275 & 0.067 & 11.70   \\ 
        BirdChicken & 0.963  & 0.41  & 0.975  & 0.49  & 0.850  & 0.850  & 8.88 & 0.900 & 0.900 & 8.14  \\ 
        DPTW & 0.756  & 73.75  & 0.751  & 85.91  & 0.589  & 0.713  & 847.07 & 0.771 & 0.723 & 265.75   \\ 
        Earthquakes & 0.780  & 58.45  & 0.783  & 67.87  & 0.745  & 0.808  & 2780.98 & 0.803 & 0.803 & 6822.34   \\ 
        ECG200 & 0.790  & 10.08  & 0.813  & 11.95  & 0.675  & 0.717  & 130.12 & 0.813 & 0.700 & 36.11 \\ \hline
    \end{tabular}}
\end{table}

For 0-dim topological features, PWGK outperforms Vrep and Vrep$_d$ in CAD, CAD$_{0.01}$, CAD$_{0.05}$ and DPTW. And SWK outperforms Vrep and $\text{Vrep}_d$ only on DPTW and ECG200. For 1-dim topological features, PWGK and SWK outperforms Vrep and Vrep$_d$ in CAD, CAD$_{0.01}$, CAD$_{0.05}$ and Earthquakes. In summary, PWGK and SWK perform similarly to Vrep and Vrep$_d$. But the time cost of PWGK and SWK is much higher than those of Vrep and Vrep$_d$.

In addition, we notice that the accuracy of Vrep and Vrep$_d$ used in a linear kernel fashion is lower than the results when they used as vector and input into a Random Forest classifier. So it is recommended that Vrep (Vrep$_d$) is used in vector form instead of kernel form.
\subsection{Comparison with more traditional baselines: Betti Curve and Euler Curve}
\label{appendix: Betti and Euler Curve}
\begin{table}[h]
    \centering
    \caption{Classification accuracy of Betti and Euler Curves. The classifier is Random Forest.}
    \label{exp: curves}
    \vspace{8pt}
    \resizebox{0.38\textwidth}{!}{
    \begin{tabular}{|c|c|c|c|c|}
    \hline
        ~ & \multicolumn{2}{c|}{Betti Curve} & \multicolumn{2}{c|}{Euler Curve}   \\ \cline{2-5}
        ~ & PD & EPD & PD & EPD  \\ \hline
        Protein & 0.95  & 0.89  & 0.21  & 0.93   \\
        CAD & $\backslash$ & 0.90  & $\backslash$ & 0.91   \\ 
        CAD$_{0.01}$ & $\backslash$ & 0.90  & $\backslash$ & 0.90   \\ 
        CAD$_{0.05}$ & $\backslash$ & 0.88  & $\backslash$ &   0.90\\ 
        Beef & 0.37  & 0.38  & 0.37  & 0.50   \\ 
        BirdChicken & 0.96  & 0.90  & 0.91  & 0.90   \\ 
        DPTW & 0.74  & 0.73  & 0.74  & 0.74   \\ 
        Earthquakes & 0.78  & 0.79  & 0.78  & 0.77   \\ 
        ECG200 & 0.78 & 0.75 & 0.83 & 0.81  \\ \hline
    \end{tabular}}
\end{table}

The results of two rather standard method Betti and Euler Curve are provided in Table \ref{exp: curves}. For Betti Curve, we choose the homology dimension to be 1. Although the time cost of Betti and Euler Curve is low, the accuracies of both Betti and Euler Curve are lower than Vrep over all the datasets.

\newpage
\subsection{Unsupervised dimensionality reduction task with Vrep on Dynamic System data}
\label{appendix: DR}

\begin{table}[h]
\centering
\caption{Dimensionality reduction results of different representation methods and corresponding Calinski-Harabasz (CH) index.}
\label{tab: DR}
\vspace{8pt}
\begin{tabularx}{\linewidth}{c X}
  \toprule
    \hfill \makecell{\begin{minipage}[c]{0.33\textwidth}
             \centering
             \raisebox{-.5\height}{\includegraphics[width=\textwidth]{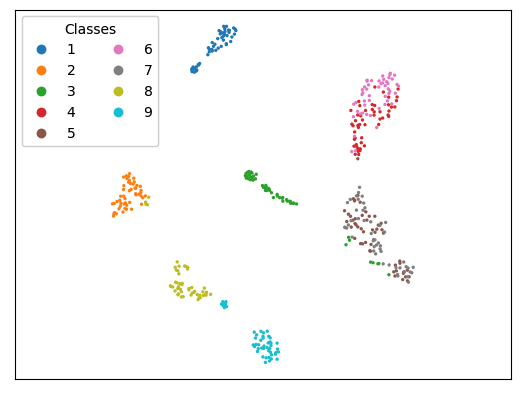}}
            \end{minipage} \\ Vrep, $\text{CH}=1751$ } 
    \hfill \makecell{\begin{minipage}[c]{0.33\textwidth}
             \centering
             \raisebox{-.5\height}{\includegraphics[width=\textwidth]{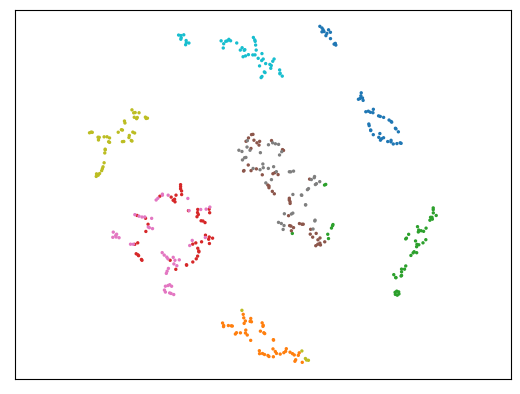}}
            \end{minipage} \\ Vrep$_d$, $\text{CH}=1702$} 
    \hfill \makecell{\begin{minipage}[c]{0.33\textwidth}
             \centering
             \raisebox{-.5\height}{\includegraphics[width=\textwidth]{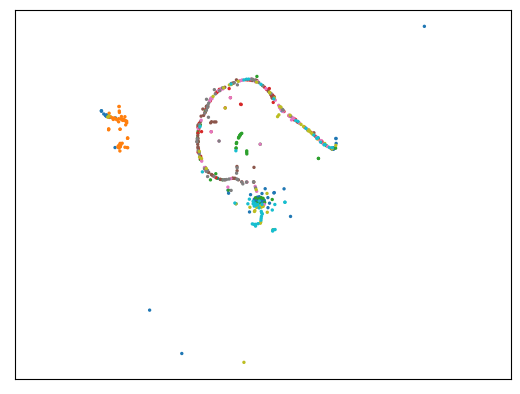}}
            \end{minipage} \\ PI, $\text{CH}=569$} 
    \hfill\null \\
    \midrule 
    \hfill \makecell{\begin{minipage}[c]{0.33\textwidth}
             \centering
             \raisebox{-.5\height}{\includegraphics[width=\linewidth]{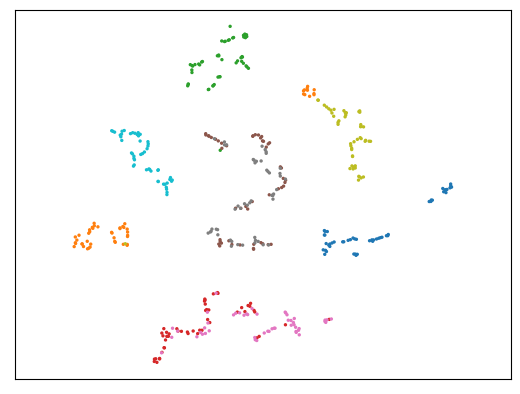}}
            \end{minipage}\\ PS, $\text{CH}=1497$  } 
    \hfill \makecell{\begin{minipage}[c]{0.33\textwidth}
             \centering
             \raisebox{-.5\height}{\includegraphics[width=\linewidth]{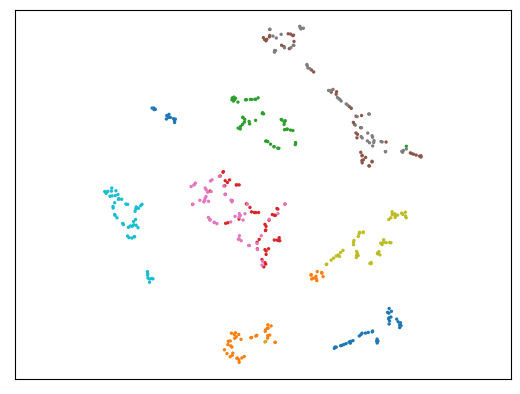}}
            \end{minipage} \\PL, $\text{CH}=1655$  } 
    \hfill\null \\
  \bottomrule
\end{tabularx}
\end{table}

Here we use Vrep, Verp$_d$, PI, PS, and PL in unsupervised dimensionality reduction task. For an EPD dataset, we obtain its high-dimensional representation with different methods and then input into t-SNE \cite{van2008visualizing} for a final 2d representation for visualization. For evaluation on the 2d repesentation, we use Calinski-Harabasz (CH) Index \cite{calinski1974dendrite}, defined as
\begin{equation*}
    CH = \frac{\sum_{i}n_i\|c_i-c\|_2/(N_C-1)}{\sum_i\sum_{x\in C_i}\|x-c_i\|_2/(n-N_C)},
\end{equation*}
where $c$ is the center of the dataset and $c_i$ is the center of cluster $C_i$. $n_i$ is the number of points in cluster $C_i$ and $N_C$ is the number of clusters. A higher CH value indicates a better cluster distribution.


We use a 3d dynamical system dataset \cite{lindstrom2002dynamics,dong2024persistence}, the same as that in the Stability Section, which describes discrete food chain model. This dataset contains 9 classes, each class contains 50 point clouds with each point cloud having 2000 points. Each class corresponds to a parameter of the dynamical system. For EPD computation, we sample $10\%$ points for sampled PD, and each EPD consists of 50 sampled PDs. Each sampled PD is computed via Rips Complex. The perplexity of t-SNE is searched in the range of $[5,10,15,20]$.

The CH result and visualization are shown in Table \ref{tab: DR}. Vrep outperforms the rest methods, followed by Vrep$_d$, which is consistent with our classification result in Table 1 of the main paper. Compared with PS and PL, Vrep and Vrep$_d$ tend to produce clusters with small intracluster distance. Part of class 2 is mixed with class 8 in PS and PL, while in Vrep and Vrep$_d$, class 2 is not split. 

\subsection{Runtime comparison with PointNet}
\label{appendix: time pointnet}
We compare the time cost of Vrep with PointNet in Table \ref{tab: time pointnet}. For small datasets, the sum of training and vectorization time is significantly smaller than the end-to-end training time of PointNet. For larger datasets like DPTW and Earthquakes, PointNet incurs a comparable time cost to Vrep because the time cost of Vrep increases quadratically with respect to the size of the dataset and the time for Vrep vectorization takes up the majority of the total time cost.


\begin{table}[h]
    \centering
    \caption{Time cost (s) of Vrep, Vrep$_d$ and PointNet. For Vrep and Vrep$_d$, the vectorization (training) column is the time cost of constructing feature (training the Random Forest classifier). PointNet obtains the hidden vector and classifier in an end-to-end manner. The end-to-end column is the total time cost of the end-to-end training process. Vrep and Vrep$_d$ are executed on CPU, while PointNet is on GPU.}
    \label{tab: time pointnet}
    \vspace{8pt}
    \begin{tabular}{|c|c|c|c|c||c|}
    \hline
    & \multicolumn{2}{c|}{Vrep} & \multicolumn{2}{c||}{Vrep$_d$} & PointNet  \\ \hline
    & vectorization & train & vectorization & train & end-to-end\\\hline
        Protein & 3.21  & 0.13  & 2.43  & 0.17  & 23.78   \\ 
        CAD & 5.55  & 0.18  & 6.34  & 0.16  & 42.77    \\ 
        CAD$_{0.01}$ & 4.25  & 0.15  & 4.96  & 0.15  & 44.35    \\ 
        CAD$_{0.05}$ & 4.20  & 0.14  & 4.93  & 0.14  & 43.42   \\ 
        Beef & 0.96  & 0.18  & 1.12  & 0.10  & 22.74    \\ 
        BirdChicken & 0.43  & 0.11  & 0.50  & 0.07  & 20.47   \\ 
        DPTW & 75.87  & 1.69  & 88.24  & 1.75  & 69.83    \\ 
        Earthquakes & 57.26  & 1.32  & 66.48  & 1.19  & 69.35   \\ 
        ECG200 & 10.16  & 0.30  & 11.91  & 0.30  & 32.78   \\ \hline
    \end{tabular}
\end{table}


\subsection{Scaleup test on the dataset size}
\label{appendix: scaleup}
We report the time ratio in the scaleup test on dataset size in Figure \ref{fig: scaleup2}, corresponding to the actual time cost in Figure 6 in the main paper.
 Vrep (Vrep$_d$) has quadratic time complexity while PI, PL and PS all have linear complexity. This is consistent with our analysis in the Time Complexity Section of the main paper.

 \begin{figure}[h]
\center
\includegraphics[width=0.45\textwidth]{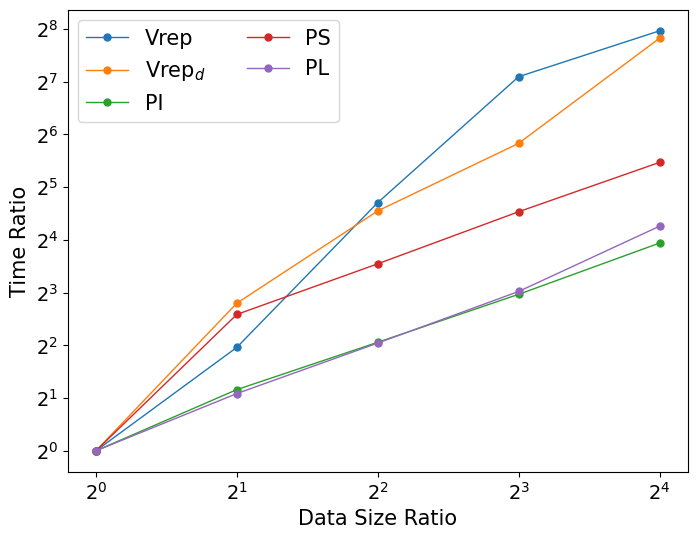}
\caption{Scaleup test on the dataset size, where the dataset size is 126 at data size ratio =1. The number $n$ of sampled PD in each EPD is set to be 10. The time ratio is reported here.}
\label{fig: scaleup2}
\end{figure}

\subsection{Synthetic Correlation with Wasserstein Distance}
\label{appendix: wass}

To isolate the approximation behavior of different vectorizations, we construct synthetic (80) EPDs with controlled distributional changes. Each EPD is an empirical measure with 120 points in the birth-death plane. We consider three settings: (i) \emph{coarse mass shift}, where one mixture component moves across separated regions of the plane; (ii) \emph{within-cell shift}, where a compact component undergoes a small local translation; and (iii) \emph{local shape change}, where the covariance orientation changes while the mean remains fixed. For each setting, we generate 80 EPDs, compute pairwise $W_1$ distances, and report Spearman's $\rho$ and Kendall's $\tau$ between $W_1$ and representation-space distances.

\paragraph{Synthetic EPD construction.}
We construct three synthetic EPD families to isolate different approximation regimes. Each EPD is an empirical measure with 120 points in the birth-death plane, and all sampled points are projected to satisfy $d>b$. The parameter $\theta\in[0,1]$ controls the distributional change within each family.

\begin{itemize}
    \item \textbf{Coarse mass shift.} Each EPD is sampled from a two-component Gaussian mixture. One component remains fixed, while the other moves as $\theta$ increases. This setting creates a coarse displacement of mass across separated regions of the birth-death plane.
    \item \textbf{Within-cell shift.} Each EPD is sampled from a single compact Gaussian component whose mean undergoes a small local translation. This setting tests whether a vectorization can detect fine local movement that may remain inside the same Voronoi cell.
    \item \textbf{Local shape change.} Each EPD is sampled from a Gaussian distribution with fixed mean but rotating anisotropic covariance. This setting changes the local shape of the distribution while keeping the coarse mass location similar.
\end{itemize}

\begin{figure}[h]
\centering
\includegraphics[width=\linewidth]{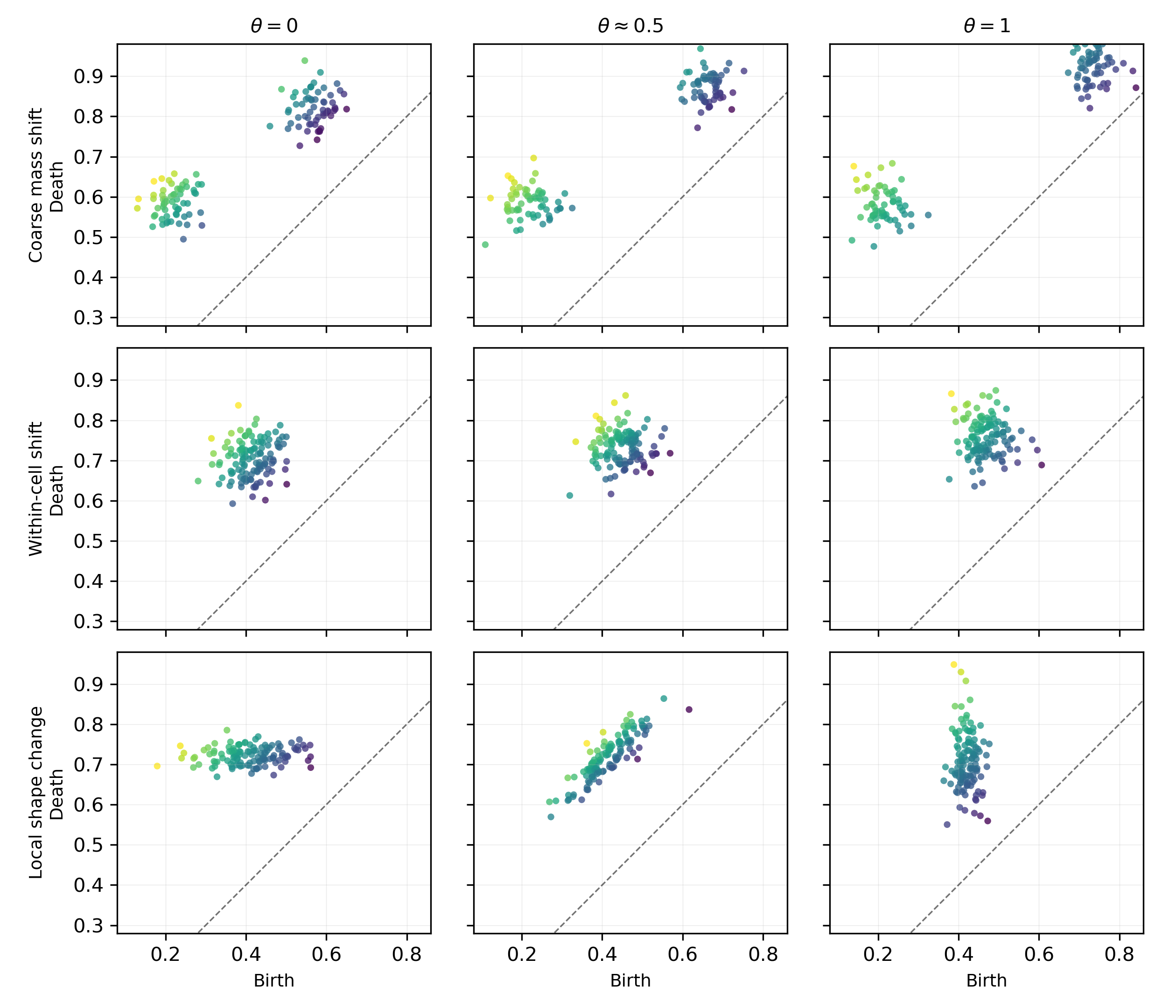}
\caption{Synthetic EPD families used in the Wasserstein correlation experiment. Columns show representative EPDs at $\theta=0$, $\theta\approx0.5$, and $\theta=1$. The coarse mass shift setting mainly changes mass allocation across separated regions; the within-cell shift setting creates small local translations; and the local shape change setting modifies within-region geometry while keeping the mean fixed.}
\label{fig:synthetic_epd_datasets}
\end{figure}

\begin{table}[h]
\centering
\caption{Synthetic correlation with Wasserstein distance. Each entry reports Spearman's $\rho$ / Kendall's $\tau$ between pairwise representation distances and $W_1$ distances.}
\label{tab: codebook new}
    \vspace{8pt}
\resizebox{0.99\textwidth}{!}{
\begin{tabular}{lccccc}
\toprule
Synthetic setting & Vrep & Vrep$_d$ & PI & PS & PL \\
\midrule
Coarse mass shift & 0.992 / 0.920 & 0.989 / 0.910 & 0.987 / 0.899 & 0.907 / 0.731 & 0.896 / 0.719 \\
Within-cell shift & 0.989 / 0.909 & 0.989 / 0.909 & 0.991 / 0.918 & 0.978 / 0.875 & 0.661 / 0.471 \\
Local shape change & 0.982 / 0.887 & 0.983 / 0.890 & 0.980 / 0.875 & 0.875 / 0.687 & 0.904 / 0.727 \\
\bottomrule
\end{tabular}}
\end{table}

The results in Table \ref{tab: codebook new} support the intended trade-off. Vrep and Vrepd obtain the highest correlations in the coarse mass shift setting, where Wasserstein variation is largely explained by mass displacement across regions that can be captured by Voronoi cells. In the within-cell shift setting, PI is slightly better, indicating that a smooth representation can be more sensitive to local translations when the Voronoi partition is already sufficiently coarse. These synthetic results therefore support the interpretation of Vrep as a complementary representation: it can preserve Wasserstein-scale variation under suitable cell resolution, but it is not a uniformly distance-preserving embedding.

\subsection{Codebook construction in Vrep/Vrep$_d$: sampling, EPD quantization, and ATOL-style centers}
\label{appendix: codebook construction}

The default Vrep construction samples codebooks from EPD supports. To test how sensitive the
empirical behavior is to this choice, we repeat the synthetic Wasserstein-correlation experiment
with two alternatives. The first alternative uses the official implementation of the EPD
quantization procedure from the GUDHI tutorial utilities (\url{https://github.com/GUDHI/TDA-tutorial/tree/master/tutorials/utils}) for
\citep{divol2021estimation}, with an additional diagonal Voronoi cell during center optimization.
For this variant, we use the same total number of codebooks and the same number of bins per
codebook as the sampled-support Vrep construction, and initialize each run from a random subset of
the pooled EPD support. This initialization avoids the first-diagram bias of the tutorial's default
high-persistence initializer in multimodal synthetic diagrams.
The second alternative uses ATOL-style global centers: we learn a global set of $k$-means centers
from the pooled training supports and use these centers as assignment anchors. In all cases, we keep
Vrep's hard Voronoi histogram readout unchanged, so the experiment isolates the effect of codebook
construction rather than changing the downstream representation.

\begin{table}[h]
\centering
\caption{Effect of codebook construction on synthetic Wasserstein correlation. Each entry reports
Spearman's $\rho$ / Kendall's $\tau$ between pairwise Vrep distances and $W_1$ distances.}
\label{tab: aaaa}
\begin{tabular}{llcc}
\toprule
Synthetic setting & Codebook construction & Vrep & Vrepd \\
\midrule
Coarse mass shift & Sampled supports & 0.992 / 0.920 & 0.989 / 0.910 \\
Coarse mass shift & EPD quantization centers & 0.993 / 0.922 & 0.992 / 0.918 \\
Coarse mass shift & ATOL-style centers & 0.969 / 0.844 & 0.962 / 0.829 \\
\midrule
Within-cell shift & Sampled supports & 0.989 / 0.909 & 0.989 / 0.909 \\
Within-cell shift & EPD quantization centers & 0.985 / 0.895 & 0.986 / 0.898 \\
Within-cell shift & ATOL-style centers & 0.941 / 0.796 & 0.945 / 0.806 \\
\midrule
Local shape change & Sampled supports & 0.982 / 0.887 & 0.983 / 0.890 \\
Local shape change & EPD quantization centers & 0.970 / 0.850 & 0.970 / 0.850 \\
Local shape change & ATOL-style centers & 0.948 / 0.800 & 0.943 / 0.793 \\
\bottomrule
\end{tabular}
\end{table}

The results in Table \ref{tab: aaaa} show that codebook construction matters. Sampled-support codebooks (the default choice of Vrep/Vrep$_d$) give the strongest
and most stable correlations in this experiment. ATOL-style centers remain competitive but are
consistently weaker, which is expected because a single global set of centers has lower resolution
than the concatenation of multiple sampled codebooks. EPD quantization is competitive on
all three settings after matching the sampled-support resolution and using a pooled-support
initializer. This comparison also shows that EPD quantization is sensitive to initialization: using
only the first diagram's high-persistence points can miss modes in multimodal diagrams and produce
poor Voronoi histograms. With a mode-covering initialization, quantized codebooks provide a
principled alternative to sampled supports, while ATOL-style centers give a lower-resolution but
still competitive global-center baseline. EPD quantization and ATOL-style have certain objective functions and hence require more time cost for optimization, while the sampled supports do not require such techniques and is faster, as shown in Table \ref{tab: ave time cost}.

\begin{table}[!ht]
    \centering
\caption{Average time cost (s).}
\label{tab: ave time cost}
    \begin{tabular}{lccc}
    \toprule
        ~ & Sample & Quantization & ATOL \\\midrule
        Coarse mass shift, Vrep & 0.0012 $\pm$ 0.0001 & 0.5031 $\pm$ 0.2604 & 0.0912 $\pm$ 0.0641 \\ 
        Coarse mass shift, Vrep$_d$ & 0.0011 $\pm$ 0.0000 & 0.3617 $\pm$ 0.0048 & 0.0592 $\pm$ 0.0036 \\ \hline
        Local shape change, Vrep & 0.0015 $\pm$ 0.0001 & 0.3900 $\pm$ 0.0083 & 0.0784 $\pm$ 0.0056 \\ 
        Local shape change, Vrep$_d$ & 0.0013 $\pm$ 0.0002 & 0.3750 $\pm$ 0.0131 & 0.0715 $\pm$ 0.0048 \\ \hline
        Within cell shift, Vrep & 0.0015 $\pm$ 0.0001 & 0.3988 $\pm$ 0.0138 & 0.0835 $\pm$ 0.0049 \\ 
        Within cell shift, Vrep$_d$ & 0.0014 $\pm$ 0.0001 & 0.3704 $\pm$ 0.0090 & 0.0830 $\pm$ 0.0089 \\ \bottomrule
    \end{tabular}
\end{table}


 \section{On the use of Large Language Models}
 \label{appendix: llm}

 Large Language Models (LLMs) were utilized in the polishing phase of this paper’s preparation. Specifically, LLMs were employed to optimize linguistic clarity, enhance stylistic coherence, and correct minor grammatical or syntactical inconsistencies. All core intellectual content, including conceptual frameworks, empirical observations, argumentative structure, and citation alignment, was developed, curated, and validated exclusively by the human authors.



\end{document}